\documentclass[11pt]{article}

\usepackage{microtype}
\usepackage{graphicx}
\usepackage{subfigure}
\usepackage{amsmath,amsbsy,amsfonts,amssymb,amsthm,bm}
\usepackage{algorithm,algorithmic,mathtools}
\usepackage{color,cases,multirow}
\usepackage[top=1.1in, bottom=1.1in, left=1.2in, right=1.2in]{geometry}

\usepackage[colorlinks=True, citecolor=blue]{hyperref}

\usepackage{authblk}

\newtheorem{theorem}{Theorem}[section]
\newtheorem{proposition}[theorem]{Proposition}
\newtheorem{lemma}{Lemma}

\newtheorem{assumption}{Assumption}
\newtheorem{corollary}[theorem]{Corollary}
\theoremstyle{definition}
\newtheorem{definition}{Definition}
\newtheorem{remark}{Remark}

\newcommand\Def{\stackrel{\text{def}}{=}}

\newcommand{\half}{\frac{1}{2}}

\newcommand{\ta}{\tilde{\bm{a}}}

\newcommand{\RR}{\mathbb{R}}
\newcommand{\EE}{\mathbb{E}}
\newcommand{\PP}{\mathbb{P}}

\newcommand{\SU}{\mathbb{S}}
\renewcommand{\SS}{\mathbb{S}}

\newcommand{\cH}{\mathcal{H}}
\newcommand{\cP}{\mathcal{P}}
\newcommand{\cF}{\mathcal{F}}

\newcommand{\cB}{\mathcal{B}}
\newcommand{\cR}{\mathcal{R}}

\newcommand{\cI}{\mathcal{I}}

\newcommand{\cO}{\mathcal{O}}
\newcommand{\cA}{\mathcal{A}}

\newcommand{\ba}{\bm{a}}
\newcommand{\bb}{\bm{b}}

\newcommand{\be}{\bm{e}}

\newcommand{\bg}{\bm{g}}

\newcommand{\bx}{\bm{x}}

\newcommand{\bu}{\bm{u}}

\newcommand{\bw}{\bm{w}}

\newcommand{\lba}{\lambda^{(a)}_n}
\newcommand{\lbb}{\lambda^{(b)}_n}
\newcommand{\ka}{k^{(a)}}
\newcommand{\kb}{k^{(b)}}
\newcommand{\Ka}{K^{(a)}}
\newcommand{\Kb}{K^{(b)}}
\newcommand{\Ga}{G^{(a)}}
\newcommand{\Gb}{G^{(b)}}
\newcommand{\bga}{\bg^{(a)}}
\newcommand{\ga}{g^{(a)}}

\newcommand{\itk}{{\it ker}}
\newcommand{\hcR}{\hat{\cR}_n}
\newcommand{\rad}{\text{Rad}}

\title{A Comparative Analysis of Optimization and Generalization Properties of 
Two-layer Neural Network and Random Feature Models Under Gradient Descent Dynamics}

\author[1,2,3]{Weinan E \thanks{weinan@math.princeton.edu}}
\author[2]{Chao Ma \thanks{cham@princeton.edu}}
\author[2]{Lei Wu \thanks{leiwu@princeton.edu}}

\affil[1]{Department of Mathematics, Princeton University}
\affil[2]{Program in Applied and Computational Mathematics, Princeton University}
\affil[3]{Beijing Institute of Big Data Research}

\date{}
\begin{document}

\maketitle

\begin{abstract}
A fairly comprehensive analysis is presented for the gradient descent dynamics for training two-layer neural network models in the situation when the parameters in both layers are updated. General initialization schemes as well as general regimes for the network width and training data size are considered. In the over-parametrized regime, it is shown that gradient descent dynamics can achieve zero training loss exponentially fast regardless of the quality of the labels. In addition, it is proved that throughout the training process the functions represented by the neural network model are uniformly close to that of a kernel method.  For general values of the network width and training data size,  sharp estimates of the generalization error  is established for target  functions in the appropriate reproducing kernel Hilbert space.
\end{abstract}

\section{Introduction}
Optimization and generalization are two central issues in the theoretical analysis of machine learning models.
These issues are of special interest for modern neural network models, not only because of their practical success \cite{krizhevsky2012a,lecun2015deep}, but also because of the fact that these neural network models are often heavily over-parametrized and traditional machine learning theory does not seem to work directly \cite{neyshabur2014search,zhang2016understanding}. For this reason, there has been a lot of recent theoretical work centered on  these issues \cite{kawaguchi2016deep,keskar2016large,du2018gradient,du2018deepgradient,allen2018convergence,cao2019generalization,daniely2017sgd,zou2018stochastic,xie2016diverse,song2018mean,rotskoff2018parameters,sirignano2018mean}.  One issue of particular interest is whether the gradient descent (GD) algorithm can produce models
that optimize the empirical risk and at the same time generalize well for the population risk.
In the case of over-parametrized two-layer neural network models, which will be the focus of this paper,
it is generally understood that as a result of the non-degeneracy of the associated Gram matrix~\cite{xie2016diverse,du2018gradient}, optimization can  be accomplished  using the gradient descent algorithm regardless of the quality of the labels, 
in spite of the fact that the empirical risk function is non-convex.
In this regard,  one can say that over-parametrization facilitates optimization.

The situation with generalization is a different story.  
There has been a lot of interest on the so-called ``implicit regularization'' effect~\cite{neyshabur2014search},
 i.e. by tuning the parameters in the optimization algorithms, one might be able to guide the algorithm to move
towards network models that generalize well, without the need to add any explicit regularization terms
(see below for a review of the existing literature).
But despite these efforts, it is fair to say that the general picture has yet to emerge.

In this paper, we perform a rather thorough analysis of the gradient descent algorithm for training two-layer neural network
models.
We study the case  in which the parameters in both the input and output layers are updated -- the case found in practice.
In the heavily over-parametrized regime, for general initializations, we prove that the results of \cite{du2018gradient}
still hold, namely, the gradient descent dynamics still converges to a global minimum exponentially fast,
regardless of the quality of the labels.
However, we also prove that the functions obtained are uniformly close to the ones found in an associated
kernel method, with the kernel defined by the initialization.

In the second part of the paper, we
 study the more general situation when the assumption of over-parametrization is relaxed.
We provide sharp estimates for both the empirical and population risks.
In particular, we prove that for target functions in the appropriate reproducing kernel Hilbert space (RKHS)~\cite{aronszajn1950theory}, the generalization
error can be made small if certain early stopping strategy is adopted for the  gradient descent algorithm.

Our results imply that under this setting over-parametrized two-layer neural networks are a lot like the kernel methods: They can always fit any set of random labels, but in order to generalize, the target functions have to be in the right RKHS.  This should be compared with the optimal generalization error bounds proved in \cite{ma2018priori} for regularized models.

\subsection{Related work}

The seminal work of \cite{zhang2016understanding} presented both numerical and theoretical evidence that over-parametrized
neural networks can fit random labels. 
Building upon earlier work on the non-degeneracy of some Gram matrices \cite{xie2016diverse}, 
Du et al. went a step further
by proving that the GD algorithm can find global minima of the empirical risk for sufficiently over-parametrized
two-layer neural networks \cite{du2018gradient}.   This result was extended to multi-layer networks in \cite{du2018deepgradient,allen2018convergence} or a general setting \cite{chizat2018note}. The related result for infinitely wide neural networks was obtained in \cite{jacot2018neural}. {In this paper, we prove
a new optimization result (Theorem \ref{thm: optimization}) that
 removes the non-degeneracy assumption of the input data  by utilizing the smoothness of the target function. Also the requirement of the network width is significantly relaxed.}

The issue of generalization is less clear. \cite{daniely2017sgd} established generalization error bounds for solutions produced by the online stochastic gradient descent (SGD) algorithm
with early stopping when the target function is in a certain RKHS.
Similar results were proved in \cite{li2018learning} for the classification  problem, in
\cite{cao2019generalization} for offline SGD algorithms, and in \cite{allen2018learning} for GD algorithm.  {These results are similar to ours, but we do not require the network to be over-parametrized. Moreover, in Theorem \ref{thm:long} we show that in this setting neural networks are uniformly close to the random feature models if the network is highly over-parametrized.}

More recently in~\cite{arora2019fine}, a generalization bound was derived for  GD solutions using 
a data-dependent norm.  This norm is bounded if the target function belongs to the appropriate RKHS. However, their error bounds are not strong enough to rule
out the possibility of curse of dimensionality. Indeed the results of the present paper do suggest that curse of dimensionality does occur in their setting (see Theorem~\ref{thm:curse_dim}).

{
\cite{jacot2018neural} provided by a heuristic argument that the GD solutions of a infinitely-wide neural network are captured by the so-called neural tangent kernel. 
In this paper, we provide a rigorous proof of the non-asymptotic version of the result for the two-layer neural network under weaker conditions.
}

\section{Preliminaries} 
Throughout this paper, we will use the following notation $[n]=\{1,2,\dots,n\}$, if $n$ is a positive integer. We use $\|\|$ and $\|\|_F$ to denote the $\ell_2$ and Frobenius norms for matrices, respectively.  We let $\SU^{d-1}=\{\bx\,:\,\|\bx\|=1\}$, and use $\pi_0$ to denote the uniform distribution over $\SU^{d-1}$.  We use $X\lesssim Y$ to indicate that there exists an absolute constant $C_0>0$ such that $X\leq C_0Y$, and  $X\gtrsim Y$ is similarly defined. If $f$ is  a function defined on $\RR^d$ and $\mu$ is a probability distribution on $\RR^d$,
we let $\|f\|_\mu = (\int_{\RR^d} f(\bx)^2 d \mu(\bx) )^{1/2}$.

\subsection{Problem setup}
We focus on the regression problem with a training data set given by $\{(\bx_i,y_i)\}_{i=1}^n$, i.i.d.\ samples drawn from a distribution $\rho$, which is assumed fixed but only known through the samples. In this paper, we assume $\|\bx\|_2= 1$ and $|y|\leq 1$. We are interested in  fitting  the data by a two-layer neural network:
\begin{equation}\label{eqn: definition-two-layer-net}
    f_m(\bx;\Theta) = \ba^T\sigma(B\bx),
\end{equation}
where $\ba\in\RR^{m}, B=(\bb_1,\bb_2,\cdots, \bb_m)^T\in \RR^{m\times d}$ and  $\Theta=\{\ba,B\}$ denote all the parameters. {Here $\sigma(t)=\max(0,t)$ is the ReLU activation function}. We will omit the subscript $m$ in the notation for $f_m$ if there is no
danger of confusion.
In formula~\eqref{eqn: definition-two-layer-net}, we omit the bias term for notational simplicity. The effect of the bias term can be incorporated if we think of $\bx$ as $(\bx, 1)^T$.

The ultimate goal is to minimize the population risk defined by
\[
    \cR(\Theta) = \frac{1}{2}\EE_{\bx,y}[(f(\bx;\Theta)-y)^2].
\]
But in practice, we can only work with the following empirical risk  
\[
    \hat{\cR}_n(\Theta)  = \frac{1}{2n}\sum_{i=1}^n (f(\bx_i;\Theta)-y_i)^2.
\]

\paragraph*{Gradient Descent}
We are interested in analyzing the property of the following gradient descent algorithm: 
$
    \Theta_{t+1} = \Theta_t - \eta \nabla\hcR(\Theta_t),
$
where $\eta$ is the learning rate. For simplicity, we will focus on its continuous version,
the gradient descent (GD) dynamics:
\begin{equation}\label{eqn: definition-GD}
    \frac{d\Theta_t}{dt} = - \nabla\hcR(\Theta_t).
\end{equation}

\paragraph*{Initialization} $\Theta_0 = \{\ba(0), B(0)\}$. We assume that  $\{\bb_k(0)\}_{k=1}^m$ are i.i.d. random variables  drawn from  $\pi_0$, and $\{a_k(0)\}_{k=1}^m$ are i.i.d. random variables drawn from the distribution defined by $\PP\{a_k(0)=\beta\}=\PP\{a_k(0)=-\beta\}=\half$. Here $\beta$ controls the magnitude of the initialization, and  it may depend on $m$, e.g. $\beta=\frac{1}{m}$ or $\frac{1}{\sqrt{m}}$.
Other initialization schemes can also be considered (e.g. distributions other than $\pi_0$, other ways of initializing
$\ba$).  The needed argument does not change much from the ones for this special case.

\subsection{Assumption on the input data}
With the activation function $\sigma(\cdot)$ and the distribution $\pi_0$, we can define two positive definite (PD) functions~\footnote{We say that a continuous symmetric function $k$ is positive definite if and only if for any $\bx_1,\dots,\bx_n$, the kernel matrix $K=(K_{i,j}) \in\RR^{n\times n}$ with $K_{i,j}=k(\bx_i,\bx_j)$ is positive definite.} 
\begin{align*}
    k^{(a)}(\bx,\bx') &\Def \EE_{\bb\sim\pi_0}[\sigma(\bb^T\bx)\sigma(\bb^T\bx')], \\
    k^{(b)}(\bx,\bx') &\Def \EE_{\bb\sim\pi_0}[\sigma'(\bb^T\bx)\sigma'(\bb^T\bx')\langle \bx,\bx'\rangle ].
\end{align*}
For a fixed training  sample, the corresponding normalized kernel matrices $K^{(a)}=(K^{(a)}_{i,j}), K^{(b)}=(K^{(b)}_{i,j})\in\RR^{n\times n}$ are defined by  
\begin{equation}\label{eqn: kernel-matrix}
    \begin{aligned}
    {K}^{(a)}_{i,j}  &=  \frac{1}{n}k^{(a)}(\bx_i,\bx_j), \\
    K^{(b)}_{i,j} &= \frac{1}{n}k^{(a)}(\bx_i,\bx_j).
    \end{aligned}
\end{equation}
Throughout this paper, we make the following assumption on the training set.
\begin{assumption}\label{assump: data}
For the given training set $\{(\bx_i,y_i)\}_{i=1}^n$, we assume that the smallest eigenvalues of the two kernel matrices defined
above  are both positive, i.e.
\[
    \lambda^{(a)}_{n} \Def \lambda_{\min}({K}^{a})>0,\quad \lambda^{(b)}_{n} \Def \lambda_{\min}(K^{(b)})>0.
\]
Let $\lambda_n = \min\{\lambda_n^a,\lambda_n^b \}$.
\end{assumption}

\begin{remark}
Note that $\lambda_n^{(a)}\leq \min_{i\in [n]}K^{(a)}_{i,i}\leq 1/n, \lambda_n^{(b)}\leq \min_{i\in [n]}K^{(a)}_{i,i}\leq 1/n$. In general, $\lba,\lbb$ depend on the data set.
For any PD functions $s(\cdot,\cdot)$, the Hilbert-Schmidt integral operator $T_s: L^2(\SU^{d-1},\pi_0)\mapsto L^2(\SU^{d-1}, \pi_0)$ is defined by  
\[
    T_{s} f(\bx) = \int_{S^{d-1}} s(\bx,\bx')f(\bx')d\pi_0(\bx').
\]
Let $\Lambda_n(T_s)$ denote its $n$-th largest eigenvalue. If $\{\bx_i\}_{i=1}^n$ are independently drawn from $\pi_0$, it was proved in \cite{braun2006accurate} that with high probability $\lba \geq \Lambda_{n}(T_{\ka})/2$ and $\lbb\geq \Lambda_n(T_{\kb})/2$. 
Using the similar idea, \cite{xie2016diverse} provided lower bounds for $\lbb$ based on some geometric discrepancy, which quantifies the uniformity degree of $\{\bx_i\}_{i=1}^n$.
In this paper, we leave $\lba>0, \lbb>0$ as our basic assumption.
\end{remark}

\subsection{The random feature model}
 We introduce the following random feature model~\cite{rahimi2008random} as 
 a reference for the two-layer neural network  model
\begin{equation}\label{eqn: definition-random-feature}
f_m(\bx; \tilde{\ba},B_0) \Def \tilde{\ba}^T \sigma(B_0\bx),
\end{equation}
where $\ba\in\RR^{m}, B_0\in\RR^{m\times d}$.
Here $B_0$ is fixed at the corresponding 
 initial values for the neural network model, and is not part of the parameters to be trained.
 The corresponding gradient descent dynamics is given by 
\begin{align}\label{eqn: refer-GD}
    \frac{d\ta_t}{dt} = -\frac{1}{n}\sum_{i=1}^n (\ta^T_t\sigma(B_0\bx_i)-y_i)\sigma(B_0\bx_i).
\end{align}
This dynamics is relatively simple since it is linear. 

\section{Analysis of the over-parameterized case}\label{sec:over_param}
In this section, we consider the optimization and generalization properties of the GD dynamics in the
over-parametrized regime. We introduce two Gram matrices $\Ga(\Theta), \Gb(\Theta)\in \RR^{n\times n}$, defined by 
\begin{align*}
    \Ga_{i,j}(\Theta) &= \frac{1}{nm}\sum_{k=1}^m \sigma(\bb_k^T\bx_i)\sigma(\bb_k^T\bx_j),\\
    \Gb_{i,j}(\Theta)&= \frac{1}{nm}\sum_{k=1}^m a_k^2 \bx_i^T\bx_j\sigma'(\bb_k^T\bx_i)\sigma'(\bb_k^T\bx_j).
\end{align*}
Let $G = \Ga+\Gb\in\RR^{n\times n}, e_j = f(\bx_j, \Theta) - y_j$ and $\be=(e_1, e_2, \cdots, e_n)^T$, it is easy to see that
\begin{equation}
    \|\nabla_{\Theta}\hcR\|^2 = \frac{m}{n}\be^T G \be.
\end{equation}
Since $\hcR = \frac{1}{2n} \be^T \be$, we have 
\[
   2m\lambda_{\min}(G)\hcR\leq  \|\nabla_{\Theta}\hcR\|^2\leq 2m \lambda_{\max}(G) \hcR.
\]

\subsection{Properties of the initialization}
\begin{lemma}\label{lem: init-risk}
For any fixed $\delta>0$, with probability at least $1-\delta$ over the random initialization,  we have
\[
\hcR(\Theta_0) \leq \frac{1}{2}\left(1+c(\delta)\sqrt{m}\beta \right)^2,
\]
where $c(\delta)=2+\sqrt{\ln(1/\delta)}$.
\end{lemma}
The proof of this lemma can be found in Appendix~\ref{sec: init-risk}.

In addition, at the initialization, the Gram matrices satisfy 
\[
    \Ga(\Theta_0)\to \Ka,\quad \Gb(\Theta_0)\to \beta^2 \Kb \quad \text{as}\,\, m\to\infty.
\]

In fact, we have
\begin{lemma}\label{lem: gram-init}
For $\delta>0$, if $m\geq \frac{8}{\lambda_n^2}\ln(2n^2/\delta)$, we have, with probability at least $1-\delta$ over the random choice of $\Theta_0$ 
\[
\lambda_{\min}(G(\Theta_0))\geq \frac{3}{4}(\lba+\beta^2\lbb).
\]
\end{lemma}
 The proof of this lemma is deferred to Appendix~\ref{sec: gram-init}.

\subsection{Gradient descent near the initialization}
\label{sec: random-label}

We define a neighborhood of the initialization by 
\begin{equation}\label{eqn: neighborhood}
    \cI(\Theta_0)\Def \{\Theta: \|G(\Theta)-G(\Theta_0)\|_F\leq  \frac{1}{4}(\lba+\beta^2\lbb)\}.
\end{equation}
Using the lemma above, we conclude that for any fixed $\delta > 0$, with probability at least $1-\delta$ over the random choices of $\Theta_0$, we must have  
\[
    \lambda_{\min}(G(\Theta))\geq \lambda_{\min} (G(\Theta_0)) - \|G(\Theta)-G(\Theta_0)\|_F \geq \frac{1}{2}(\lba+\beta^2\lbb)
\]
for all $\Theta \in \cI(\Theta_0)$.

For the GD dynamics, we define the exit time of $\cI(\Theta_0)$ by
\begin{equation}\label{eqn: def-t0}
    t_0 \Def \inf \{\,t \,: \Theta_t\notin \cI(\Theta_0)\}.
\end{equation}
\begin{lemma}\label{pro: 1}
For any fixed $\delta\in (0,1)$, assume that $m\geq\frac{8}{\lambda_n^2}\ln(2n^2/\delta)$. Then
 with probability at least $1-\delta$ over the random choices of $\Theta_0$,  we have the following holds for any $t\in [0,t_0]$,
\[
    \hat{\cR}_n(\Theta_t)\leq e^{-m(\lba+\beta^2\lbb)t}\hat{\cR}_n(\Theta_0).
\]
\end{lemma}
\begin{proof}
We have 
\[ 
    \frac{d\hat{\cR}_n(\Theta_t)}{dt} = -\|\nabla_{\Theta}\hat{\cR}_n\|^2_F\leq -m(\lba+\beta^2\lbb) \hat{\cR}_n(\Theta_t),
\]
where the last inequality is due to the fact that $\Theta_t\in\cI(\Theta_0)$. This completes the proof.
\end{proof}

We define two quantities:
\begin{equation}\label{eqn: def-pn-qn}
p_n \Def \frac{4\sqrt{\hcR(\Theta_0)}}{m(\lba + \beta^2 \lbb)}, \quad  q_n \Def p_n^2 + \beta p_n.
\end{equation}
The following is the most crucial characterization of the GD dynamics.
\begin{proposition}\label{pro: pertubration}
For any $\delta>0$, assume $m\geq 1024 \lambda_n^{-2} \ln(n^2/\delta)$. Then, with probability at least $1-\delta$,  we have the following holds for any $t\in [0,t_0]$,
\begin{align*}
    |a_k(t)-a_k(0)|&\leq  2p_n\\ 
    \|\bb_k(t)-\bb_k(0)\|& \leq 2q_n.
\end{align*}
\end{proposition}
\begin{proof}
First, we have  
\begin{align*}
    \|\nabla_{a_k}\hcR\|^2 &= \big(\frac{1}{n} \sum_{i=1}^n e_i \sigma(\bx_i^T\bb_k)\big)^2\leq 2 \|\bb_k\|^2 \hcR(\Theta), \\
    \|\nabla_{\bb_k}\hcR\|^2&= \|\frac{1}{n} \sum_{i=1}^n e_i a_k \sigma'(\bx_i^T\bb_k)\bx_i\|^2 \leq 2a_k^2\hcR(\Theta).
\end{align*}
 To facilitate the analysis, we define the following two quantities,
\begin{align*}
    \alpha_k(t)  = \max_{s\in [0,t]} |a_k(s)|, \quad
    \omega_k(t)   = \max_{s\in [0,t]} \|\bb_k(s)\|.
\end{align*}
Using Lemma \ref{pro: 1}, we have
\begin{equation}
\begin{aligned}\label{eqn: -1-1-1}
\|\bb_k(t)-\bb_k(0)\| &\leq \int_0^t \|\nabla_{\bb_k} \hat{\cR}_n(\Theta_{t'})\| dt' \\
&\leq 2\int_0^t \alpha_k(t) \sqrt{\hat{\cR}_n(\Theta_{t'})} dt' \\
& \leq \frac{4 \sqrt{\hat{\cR}_n(\Theta_0)} \alpha_k(t) }{m(\lba+\beta^2\lbb)} = p_n \alpha_k(t), \\
|a_k(t) - a_k(0)|&\leq \int_0^t |\nabla_{a_k} \hcR(\Theta_{t'})| dt' \\
&\leq 2\int_0^t \omega_k(t)\sqrt{\hcR(\Theta_{t'})} dt' \\
&\leq \frac{4 \sqrt{\hat{\cR}_n(\Theta_0)}\omega_k(t)}{m(\lba+\beta^2 \lbb)} = p_n \omega_k(t).
\end{aligned}
\end{equation}
Combining the two inequalities above, we get
\begin{equation*}
    \alpha_k(t)\leq |a_k(0)|+p_n\left(1 + p_n \alpha_k(t)\right).
\end{equation*}
Using Lemma~\ref{lem: init-risk} and the fact that  $m\geq \max\{\frac{16}{\lba},\frac{64c^2(\delta)}{\lbb\lba}\}$, we have
\begin{align}\label{eqn: p_n}
\nonumber   p_n &\leq \frac{4(1+c(\delta)\sqrt{m}\beta )}{m(\lba + \beta^2\lbb)} \\
 &\leq \frac{4}{m\lba} + \frac{4c(\delta)}{\sqrt{m\lba\lbb}} \leq  \half.
\end{align}
Therefore,
\begin{align*}
\alpha_k(t)\leq (1-p_n^2)^{-1}(p_n + \beta)\leq 2(p_n+\beta).
\end{align*}
Inserting the above estimates back to~\eqref{eqn: -1-1-1}, we  obtain 
\begin{align*}
    \|\bb_k(t)-\bb_k(0)\|& \leq 2p_n^2 + 2\beta p_n.
\end{align*}
Since $m\geq \max\{\frac{16}{\sqrt{\lba\lbb}},\frac{1024c^2(\delta)}{(\lbb)^2}\}$, we have
\begin{align}\label{eqn: q_n}
\nonumber 2\beta p_n &\leq \frac{8\beta(1+c(\delta)\sqrt{m}\beta)}{m(\lba + \beta^2 \lbb)} \\
\nonumber &\leq \frac{8\beta}{m(\lba+\beta^2\lbb)} + \frac{8c(\delta)}{\sqrt{m}\lbb}\\
\nonumber &\leq \frac{4}{m\sqrt{\lba\lbb}} + \frac{8c(\delta)}{\sqrt{m}\lbb}\\
&\leq \frac{1}{2}.
\end{align}
 Therefore we have $\omega_k(t)\leq 1 + \|\bb_k(t)-\bb_k(0)\|\leq  2$, which leads to 
\[
    |a_k(t)-a_k(0)|\leq p_n \omega_k(t)\leq 2p_n.
\]
\end{proof}

The following lemma provides that how $p_n$ and $q_n$ depend on $\beta$ and $m$.
\begin{lemma}\label{lemma: dist-a-b}
For any $\delta>0$, assume $m\geq 1024 \lambda_n^{-2} \ln(n^2/\delta)$. Let $C(\delta) = 10c^2(\delta)$. 
If $\beta\leq 1$, we have 
\begin{equation}\label{eqn: 000}
\begin{aligned}
    p_n &\leq \frac{C(\delta)}{\sqrt{m}\lba}\left(\frac{1}{\sqrt{m}}+\beta\right) \\
    q_n &\leq \frac{C(\delta)}{m(\lba)^2}\left(\frac{1}{m}+\frac{2\beta}{\sqrt{m}}+\beta^2\right) + \frac{C(\delta)\beta}{m\lba} +\frac{ C(\delta) \beta^2}{\sqrt{m}\lba}.
\end{aligned}
\end{equation}
If $\beta > 1$, we have 
\begin{equation}\label{eqn: beta-large}
\begin{aligned}
p_n &\leq \frac{C(\delta)}{\sqrt{m\lba\lbb}} \\
q_n & \leq \frac{C(\delta)}{\sqrt{m}\lbb}.
\end{aligned}
\end{equation}
\end{lemma}

\subsection{Global convergence for arbitrary labels}
Proposition~\ref{pro: pertubration} and Lemma~\ref{lemma: dist-a-b} tell us that no matter how large  $\beta$ is,  we have 
\[
    \max_{k\in [m]}\big\{\|\bb_k(t)-\bb_k(0)\|,|a_k(t)-a_k(0)|\big\} \to 0 \quad \text{as }\, m\to \infty.
\]
This actually implies that the GD dynamics always stays in $\cI(\Theta_0)$, i.e.  $t_0=\infty$. 
\begin{theorem}\label{thm: optimization}
For any $\delta\in (0,1)$,
assume $m\gtrsim \lambda_n^{-4}n^2\delta^{-1}\ln(n^2/\delta)$. Then with probability at least $1-\delta$ over the random
 initialization, we have 
 \[
  \hat{\cR}_n(\Theta_t)\leq e^{-m(\lba+\beta^2\lbb)t}\hat{\cR}_n(\Theta_0),
 \]
 for any $t\geq0$.
\end{theorem}
{
\begin{proof}
According to Lemma~\ref{pro: 1}, we only need to prove that $t_0=\infty$.
Assume $t_0<\infty$. 
 
 Let us first consider  the Gram matrix $\Ga$. Since $\sigma(\cdot)$ is $1-$Lipschitz and $\max_{k}\|\bb_k(t_0)-\bb_k(0)\|\leq q_n \leq 1$, we have 
\begin{align*}
|\Ga_{i,j}(\Theta_{t_0})-\Ga_{i,j}(\Theta_0)|&= \frac{1}{nm}\sum_{k=1}^m\big(\sigma(\bb^T_k(t_0)\bx_i)\sigma(\bb^T_k(t_0)\bx_j)- \sigma(\bb^T_k(0)\bx_i)\sigma(\bb^T_k(0)\bx_j)\big)\\
&\leq \frac{1}{nm}\sum_{k=1}^m \left(2\|\bb_k(t_0)-\bb_k(0)\| + \|\bb_k(t_0)-\bb_k(0)\|^2\right)\\
& \leq \frac{3q_n}{n}.
\end{align*}
This leads to
\begin{equation}\label{eqn: Ga-p}
\|\Ga(\Theta_{t_0})-\Ga(\Theta_0)\|_F\leq 3 q_n.
\end{equation}

 Next we turn to the Gram matrix $\Gb$. Define the event
\[
    D_{i,k} = \{\,\bb_k(0) : \|\bb_k(t_0)-\bb_k(0)\|\leq q_n, \sigma'(\bb^T_k(t_0)\bx_i)\neq \sigma'(\bb^T_k(0)\bx_i)\}.
\]
Since $\sigma(\cdot)$ is ReLU, this event happens  only if $|\bx_i^T\bb_k(0)|\leq q_n$. 
By the fact that $\|\bx_i\|=1$ and $\bb_k(0)$ is drawn from the uniform distribution over the sphere, we have $\PP[D_{i,k}]\lesssim q_n$. Therefore the entry-wise deviation of $\Gb$ satisfies, 
\begin{align*}
n|\Gb_{i,j}(\Theta_{t_0}) &-\Gb_{i,j}(\Theta_0)| \\
&\leq \frac{|\bx_i^T\bx_j |^2}{m^2}|\sum_{k=1}^m \left(a_k^2(t_0)\sigma'(\bb^T_k(t_0)\bx_i)\sigma'(\bb^T_k(t_0)\bx_j)- a^2_k(0)\sigma'(\bb^T_k(0)\bx_i)\sigma'(\bb^T_k(0)\bx_j)\right)|\\
&\leq \frac{1}{m^2} | \sum_{k=1}^m \left(a^2_k(t_0) Q_{k,i,j} + P_k\right)|,
\end{align*}
where 
\begin{align*}
    Q_{k,i,j} &= |\sigma'(\bx_i^T\bb_k(t_0))\sigma'(\bx_j^T\bb_k(t_0)) - \sigma'(\bx_i^T\bb_k(0))\sigma'(\bx_j^T\bb_k(0))| \\
    P_k &= |a^2_k(t_0)-a^2_k(0)|.
\end{align*}
Note that $\EE[Q_{k,i,j}]\leq \PP[D_{k,i}\cup D_{k,j}]\lesssim q_n$. 
In addition, by Proposition~\ref{pro: pertubration}, we have 
\begin{align*}
P_k&\leq (\beta+2p_n)^2-\beta^2 \lesssim q_n\\
a^2_k(t_0) &\leq a^2_k(0) + P_k \lesssim \beta^2 + q_n.
\end{align*}

Hence using $q_n=p_n^2+\beta p_n\leq 1$, we obtain 
\begin{align}\label{eqn: Gb-dev}
\nonumber    n\EE[|\Gb_{i,j}(\Theta_{t_0})-\Gb_{i,j}(\Theta_0)|] &\lesssim (\beta^2+q_n)q_n + q_n\\
&\lesssim (1+\beta^2)q_n.
\end{align}
By the Markov inequality, with probability $1-\delta/n$ we have 
\[
    |\Gb_{i,j}(\Theta_{t_0})-\Gb_{i,j}(\Theta_0)|\leq \frac{(1+\beta^2)q_n}{\delta}.
\]
Consequently, with probability $1-\delta$ we have
\begin{equation}\label{eqn: Gb-p}
    \|\Gb(\Theta_{t_0})-\Gb(\Theta_0)\|_F \lesssim \frac{(1+\beta^2)nq_n}{\delta}.
\end{equation}

Combining \eqref{eqn: Ga-p} and \eqref{eqn: Gb-p}, we get 
\begin{align*}
 \|G(t_0)-G(0)\|_F&\leq  \|\Ga(t_0)-\Ga(0)\|_F + \|\Gb(t_0)-\Gb(0)\|_F\\
                &\lesssim 3q_n +\frac{(1+\beta^2)nq_n}{\delta} \\
                &\lesssim \frac{(n\delta^{-1}+1)C(\delta)}{\sqrt{m}\lbb} + \beta^2 \frac{n\delta^{-1}C(\delta)}{\sqrt{m}\lbb},
\end{align*}
where the last inequality comes from Lemma~\eqref{lemma: dist-a-b}.
Taking $m\gtrsim \lambda_n^{-4}n^2\delta^{-1}\ln(n^2/\delta)$, we get 
\[
\|G(t_0)-G(0)\|_F < \frac{1}{4}(\lba+\beta^2\lbb).
\]
The above result contradicts the definition of $t_0$. Therefore $t_0=\infty$.
\end{proof}


\begin{remark}
Compared with Proposition~\ref{pro: pertubration}, the above theorem imposes a stronger assumption on the network width: $m\geq \text{poly}(\delta^{-1})$. This is due to the lack of continuity of $\sigma'(\cdot)$ when handling $\|\Gb(\Theta_{t_0})-\Gb(\Theta_0)\|_F$. If $\sigma'(\cdot)$ is continuous, we can get rid of the dependence on $\text{poly}(\delta^{-1})$.  In addition, it is also possible to remove this assumption for the case when $\beta=o(1)$, since in this case
the Gram matrix $G=\Ga+\beta^2\Gb$ is dominated by $\Ga$.
\end{remark}

}

\begin{remark}
Theorem~\ref{thm: optimization}  is closely related to the result of Du et al. \cite{du2018gradient} where exponential convergence to global minima was first proved for over-parametrized two-layer neural networks. But it improves the result of \cite{du2018gradient} in two aspects.
 First, as is done in practice, we allow the parameters in both layers to be updated,  while \cite{du2018gradient}  chooses to freeze the parameters in the first layer. Secondly, our analysis does not impose any specific requirement on the scale of the initialization whereas the proof of \cite{du2018gradient} relies on the specific  scaling: $\beta\sim 1/\sqrt{m}$. 
\end{remark}

\subsection{Characterization of the whole GD trajectory}
In the last subsection, we showed that  very wide networks can fit arbitrary labels. 
In this subsection, we study the functions represented by such networks. We show that for highly over-parametrized two-layer
neural networks, the solution of the GD dynamics is uniformly close to the solution for  the random feature model starting from the same initial function. 


\begin{theorem}\label{thm:long}  Assume $\beta \le 1$.
Denote the solution of GD dynamics for the random feature model by 
\[
f_m^{\itk}(\bx,t)= f_m(\bx;\tilde{\ba}_t,B_0),
\]
where $\tilde{\ba}_t$ is the solution of GD dynamics~\eqref{eqn: refer-GD}. For any $\delta\in (0,1)$, assume that $m\gtrsim\lambda_n^{-4}n^2\delta^{-1}\ln(n^2\delta^{-1})$.  Then with probability at least $1-6\delta$ we have,
\begin{equation}\label{eqn: diff-nn-rf}
|f_m(\bx;\Theta_t) - f_m^\itk(\bx,t) |\lesssim \frac{c^2(\delta)}{\lba}\left(\frac{1}{\sqrt{m}}+\beta+\sqrt{m}\beta^3\right),
\end{equation}
where $c(\delta)= 1+\sqrt{\ln(1/\delta)}$.
\end{theorem}

\begin{remark}
Again the factor $\delta^{-1}$ in the condition for $m$ can be removed if $\sigma$ is assumed to be smooth or $\beta$ is assumed to be small
(see the remark at the end of Theorem \ref{thm: optimization}).
\end{remark}

\begin{remark}
If $\beta=o(m^{-1/6})$, the right-hand-side of \eqref{eqn: diff-nn-rf} goes to $0$ as  $m \rightarrow \infty$. 
For example, if we take $\beta=1/\sqrt{m}$, we have
\begin{equation}
|f_m(\bx;\Theta_t)-f_m^\itk(\bx,t)\|\lesssim \frac{c(\delta)}{\lba \sqrt{m}}.
\end{equation}
Hence this theorem says that the GD trajectory of a very wide network is uniformly close to the GD trajectory of the related kernel method~\eqref{eqn: refer-GD}. 
\end{remark}

\subsubsection*{Proof of Theorem~\ref{thm:long}}
We define
\begin{equation}
\begin{aligned}
g^{(a)}(\bx,\bx') &= \frac{1}{mn}\sum\limits_{k=1}^m \sigma(\bb_k(0)^T\bx)\sigma(\bb_k(0)^T\bx')\\
g(\bx,\bx',t) &= \frac{1}{mn}\sum\limits_{k=1}^m \left(\sigma(\bb_k(t)^T\bx)\sigma(\bb_k(t)^T\bx')+a_k(t)^2\sigma'(\bb_k(t)^T\bx)\sigma'(\bb_k(t)^T\bx')\bx^T\bx'\right).
\end{aligned}
\end{equation}
Recall the definition of $G(\Theta_t)$ in Section~\ref{sec:over_param}, we know that $G(\Theta_t)_{i,j}=g^m(x_i,x_j,t)$. 
For any $\bx\in \SU^{d-1}$, let $\bg(\bx,t)$ and $\bg^{(a)}(\bx)$ be two $n$-dimensional vectors defined by  
\begin{equation}
\begin{aligned}
\bg^{(a)}_i(\bx) &= \bga(\bx,\bx_i)\\
\bg_i(\bx,t) & =g(\bx,\bx_i,t).
\end{aligned}
\end{equation}
For GD dynamics \eqref{eqn: definition-GD}, define $\be(t)=(f_m(\bx;\Theta_t)-y_i)\in\RR^n$. Then we have,
\begin{equation}\label{eq:dynamic_nn}
\begin{aligned}
&\frac{d}{dt}\be(t) =-mG(\Theta_t)\be(t) \\
&\frac{d}{dt}f_m(\bx;\ba_t,B_t) = -m\bg(\bx,t)^T\be(t).
\end{aligned}
\end{equation}
For GD dynamics~\eqref{eqn: refer-GD} of the random feature model, we define $\tilde{\be}(t)=(f_m(\bx_i;\tilde{\ba}_t,B_0)-y_i)\in\RR^{n}$. Then, we have 
\begin{equation}\label{eq:dynamic_rf}
\begin{aligned}
&\frac{d}{dt}\tilde{\be}(t)=-m\Ga(\Theta_0)\tilde{\be}(t)\\
&\frac{d}{dt}f_m(\bx;\tilde{\ba}_t,B_0) = -m\bga(\bx)^T\tilde{\be}(t).
\end{aligned}
\end{equation}
From \eqref{eq:dynamic_nn} and \eqref{eq:dynamic_rf}, we have
\begin{equation}\label{eqn: nn-rf-express}
\begin{aligned}
f_m(\bx;\ba_t,B_t) & =f_m(\bx;\ba_0,B_0)- m\int_0^t \bg(\bx,s)^T\be(s)ds,\\
f_m(\bx;\tilde{\ba}_t,B_0) & =f_m(\bx;\ba_0,B_0)- m\int_0^t \bga(\bx)^T\tilde{\be}(s)ds,
\end{aligned}
\end{equation}
Let 
\begin{eqnarray*}
&& J_1(\bx,t)=m\int_0^t(\bg(\bx,s)-\bga(\bx))^T\be(s)ds, \\
&& J_2(\bx,t)=m\int_0^t \bga(\bx)^T\left(\be(s)-\tilde{\be}(s)\right)ds,
\end{eqnarray*}
then we have
\begin{equation}
f_m(\bx;\Theta_t)=f_m(\bx;\tilde{\ba}_t,B_0)+J_1(\bx,t)+J_2(\bx,t).
\end{equation}
We  are now going to bound $J_1(\bx,t)$ and $J_2(\bx,t)$.

 We first consider $J_1$. By Theorem~\eqref{pro: pertubration}, with probability at least $1-\delta$ we have
\begin{align*}
|g(\bx,\bx',t)-g^{(a)}(\bx,\bx')| &\leq \frac{3q_n}{n} + \frac{(\beta+2p_n)^2}{n} \\
&\lesssim \frac{\beta^2+q_n}{n} 
\end{align*}
for any $t\geq0$. Therefore, for any $\bx \in \SU^{d-1}$, we have
\begin{align*}
|J_1(\bx,t)| &\leq m\int_0^t \| \bg(\bx,s)-\bga(\bx)\|\|\be(s)\|ds\nonumber \\
 &\leq  m\left(\frac{3q_n+\beta^2}{\sqrt{n}}\right)\int_0^t\|\be(s)\|ds \nonumber \\
 &\leq  m\left(3q_n+\beta^2\right)\int_0^t\sqrt{\hcR(\Theta_s)}ds \nonumber \\
 &\lesssim \frac{q_n+\beta^2}{\lba+\beta^2\lbb}\sqrt{\hcR(\Theta_0)}.
\end{align*}
Hence, by the estimates of $\hcR(\Theta_0)$ in Lemma~\ref{lem: init-risk}, we have 
\begin{equation}
|J_1(\bx,t)|\lesssim \frac{q_n+\beta^2}{\lba}\left(1+c(\delta)\sqrt{m}\beta\right).
\end{equation}
Inserting the estimate of $q_n$ in Lemma~\ref{lemma: dist-a-b}, we get 
\begin{equation}\label{eqn: J1-bound}
|J_1(\bx,t)|\lesssim\frac{c^2(\delta)}{\lba}\left(\frac{1}{\sqrt{m}}+\beta+\sqrt{m}\beta^3\right).
\end{equation}

 Next we turn to estimating $J_2$. Let $\bu(t)=\be(t)-\tilde{\be}(t)$. 
 Following \eqref{eq:dynamic_nn} and \eqref{eq:dynamic_rf}, we obtain 
\begin{align*}
\bu(0) &= 0 \\
\frac{d}{dt}\bu(t) &=-m \Ga(\Theta_0) \bu(t)+m(\Ga(\Theta_0)-G(\Theta_t))\be(t).
\end{align*}
Solving the equation above gives
\begin{equation}
\bu(t) = m\int_0^t e^{-m \Ga(\Theta_0) (t-s)}(\Ga(\Theta_0)-G(\Theta_s))\be(s)ds.
\end{equation}
Consider the initializations for which  $\lambda_{\min}(\Ga(\Theta_0))\geq \frac{3\lba}{4}$. 
The probability of this event is no less than $1-\delta$. For such initializations, we have
\begin{equation}
\|\bu(t)\|\leq m\int_0^t e^{-\frac{3}{4}m\lba(t-s)}\|\Ga(\Theta_0)-G(\Theta_s)\|_F \|\be(s)\|ds.
\end{equation}
Using Proposition~\eqref{pro: pertubration}, we conclude that with probability no less than $1-2\delta$,  the following holds:
\begin{align}
\|G(\Theta_s)-\Ga(\Theta_0)\|_F &\leq  \|\Ga(\Theta_s)-\Ga(\Theta_0)\|_F + \max_{k\in [m]} a^2_k(s) \\
&\leq  3q_n + (\beta+2p_n)^2 \\
&\lesssim q_n +\beta^2.
\end{align}
Together with the fact that $\|\be(s)\|\leq \sqrt{2n\hat{\cR}_n(\Theta_0)}e^{-\frac{m\lba}{2}s}$, we  obtain
\begin{align} \label{eq:ut}
\nonumber \|\bu(t)\| &\lesssim m (\beta^2+q_n)\sqrt{2n\hcR(\Theta_0)} \int_0^t e^{-\frac{3}{4}m\lba(t-s)}e^{-\frac{m \lba}{2}s}ds  \\
\nonumber &\leq m (\beta^2+q_n)\sqrt{2n\hcR(\Theta_0)} \int_0^t e^{-\frac{1}{4}m\lba(t-s)}e^{-\frac{m \lba}{2}s}ds  \\
&\lesssim m (\beta^2+q_n)\sqrt{2n\hcR(\Theta_0)} \frac{1}{\lba} e^{-\frac{m\lba}{4} t}
\end{align}
In addition, for any $\bx\in\SS^{d-1}$, we have $\|\bga(\bx)\|\leq \frac{1}{m\sqrt{n}}$. Hence, plugging~\eqref{eq:ut} into $J_2$ leads to 
\begin{align}
|J_2(\bx,t)| &\leq m\int_0^t \|\bga(\bx)\|\|\bu(s)\|ds \nonumber \\
 &\lesssim \frac{m(\beta^2+q_n)\sqrt{\hcR(\Theta_0)}}{\lba}\int_0^t e^{-\frac{m\lba}{4}s}ds \nonumber \\
 &\lesssim (\beta^2+q_n)\sqrt{\hcR(\Theta_0)}.
\end{align}
Substituting in the estimates for $q_n$ and $\hcR(\Theta_0)$, and assuming that $\beta^2\leq1$, we obtain,
for any $\delta>0$, with probability no less than $1-3\delta$,
\begin{equation}
|J_2(\bx,t)|\lesssim \frac{c^2(\delta)}{\lba} \left(\frac{1}{\sqrt{m}}+\beta+\sqrt{m}\beta^3\right).
\end{equation} 

Finally, combining the estimates of $J_1$ and $J_2$, we conclude that
\begin{align}
|f_m(\bx;\Theta_t) -f_m^\itk(\bx, t)|  \lesssim \frac{c^2(\delta)}{\lba}\left(\frac{1}{\sqrt{m}}+\beta+\sqrt{m}\beta^3\right), \label{eq:estimation}
\end{align}
holds for any $\delta>0$ with probability at least $1-6\delta$.
This completes the proof of Theorem \ref{thm:long}.

\subsection{Curse of dimensionality of the implicit regularization}

From~\eqref{eqn: nn-rf-express},  we have
\begin{align}\label{eqn: rf-express}
\nonumber f_m^{\itk}(\bx,t) &= f_m(\bx;\Theta_0) - m\bga(\bx) \int_0^t \tilde{\be}(s) ds \\
&= f_m(\bx;\Theta_0) - \sum_{i=1}^n \ga(\bx,\bx_i) w_i(t),
\end{align}
where $w_i(t) = m\int_0^t \tilde{e}_i(s)ds$. 
The second term in the  right hand slide of \eqref{eqn: rf-express} lives in the span of $n$ fixed basis: $\{\ga(\bx,\bx_1), \ga(\bx,\bx_2), \cdots, \ga(\bx,\bx_n)\}$. 

For any probability distribution $\pi$ over $\SS^{d-1}$, we define
\begin{equation*}
\cH_\pi=\left\{ \int_{\SS^{d-1}}a(\bw)\sigma(\bw^T\bx) \pi(d\bw) :  \ \ \int_{\SU^{d-1}}a^2(\bw)\pi(d\bw)<\infty \right\}.
\end{equation*}
For any $h\in \cH_\pi$, define $\|h\|^2_{\cH_{\pi}}=\EE_{\pi}[|a^2(\bw)|]$.  As shown in \cite{rahimi2008uniform}, $\cH_{\pi}$ is exactly the RKHS with the kernel defined by $k(\bx,\bx')=\EE_{\pi}[\sigma(\bw^T\bx)\sigma(\bw^T\bx')]$.

\begin{definition}[Barron space]\label{def: Barron-space}
The Barron space is defined as the union of $\cH_{\pi}$, i.e.
\[
    \cB \Def \cup_{\pi} \cH_{\pi}.
\]
The Barron norm for any $h\in \cB$ is defined by 
\[
    \|h\|_{\cB} \Def \inf_{\pi} \|h\|_{\cH_{\pi}}.
\]
\end{definition}

To signify the dependence on the target function and data set, we introduce the notation:
\begin{align}\label{eqn: GD-estimator}
\cA_t(f,\{\bx_i\}_{i=1}^n,\Theta_0) = f_m^\itk(\cdot,t).
\end{align}
where the right hand side is the GD solution of the random feature model obtained by using 
the training data $\{\bx_i, y_i\}_{i=1}^n$  with $y_i = f(\bx_i)$
and $\Theta_0$ as the initial parameters. Let $\cB_Q=\{f\in\cB\,:\, \|f\|_\cB\leq Q\}$. We then have the following theorem.


\begin{theorem}\label{thm:curse_dim}
There exists an absolute constant $\kappa>0$, such that  for any $t\in [0,+\infty)$
\begin{equation}
\sup_{f\in\cB_Q} \|f-\cA_t(f,\{\bx_i\}_{i=1}^n,\Theta_0) \|_{\rho}\geq \frac{\kappa Q}{d (n+1)^{1/d}}.
\end{equation}
\end{theorem}
\begin{remark}
Combined Theorem \ref{thm:curse_dim} with Theorem \ref{thm:long}, we conclude that for any $\delta\in (0,1)$, if $m$ is
sufficiently large, then with probability at least $1-\delta$ we have
\begin{equation}\label{eqn: implicit}
\sup_{f\in\cB_Q} \|f-\cB_t(f,\{\bx_i\}_{i=1}^n,\Theta_0) \|_{\rho}\geq \frac{\kappa Q}{d (n+1)^{1/d}}
-\frac{c^2(\delta)}{\lba}\left(\frac{1}{\sqrt{m}}+\beta+\sqrt{m}\beta^3\right),
\end{equation}
where 
\[
\cB_t(f,\{\bx_i\}_{i=1}^n,\Theta_0) = f_m(\cdot, \Theta(t))
\]
denotes the solution at time $t$ of the GD dynamics for the two-layer neural network model.
If $\beta$ is sufficiently small (e.g. $\beta = o(m^{-1/6})$), then we see that the curse of dimensionality also
holds for the solutions generated by the GD dynamics for the two-layer neural network model.
Since this statement holds for all time $t$,
no early-stopping strategy is able to fix this curse of dimensionality problem. 

In contrast, it has been proved in \cite{ma2018priori} that an appropriate regularization can  avoid this curse of dimensionality
problem, i.e.
if we denote by $\mathcal{M}(f,\{\bx_i\}_{i=1}^n)$ the estimator for the regularized model in  \cite{ma2018priori}
(see  \eqref{eqn: path-reg} below),
then it was shown that for any $\delta > 0$,   with probability at least $1-\delta$ over the sampling of $\{\bx_i\}_{i=1}^n$, the following holds
\begin{equation}\label{eqn: 345}
    \sup_{f\in\cB_Q} \|f-\mathcal{M}(f,\{\bx_i\}_{i=1}^n) \|_{\rho}\lesssim \frac{Q}{\sqrt{n}}\left(\sqrt{\ln(d)} + \sqrt{\ln(n/\delta)}\right).
\end{equation}
{ 
The comparison between \eqref{eqn: implicit} and \eqref{eqn: 345} provides a quantitative understanding  of the insufficiency of using the random feature model to explain the generalization behavior of neural network models.}
\end{remark}

To prove Theorem~\ref{thm:curse_dim}, we need the following lemma, which is proved in~\cite{barron1993universal}.
\begin{lemma}\label{lm:Barron}
Let $\hat{f}(\omega)$ be the Fourier transform of  a function $f$ defined on $\RR^d$. 
Let $\Gamma_Q=\{f:\ \int \|\omega\|_1^2 |\hat{f}(\omega)|d\omega<Q\}$. 
Then, for any fixed functions $h_1, h_2, ..., h_n$, we have
\begin{equation}
\sup_{f\in\Gamma_Q}\inf_{h \in \text{span}(h_1,...,h_n)} \|f-h\| \geq \frac{\kappa Q}{dn^{1/d}}.
\end{equation}
\end{lemma}

We now  prove Theorem~\ref{thm:curse_dim}.
\begin{proof}
As is shown in~\cite{breiman1993hinging,klusowski2016risk}, any function $f\in\Gamma_Q$ can be represented as 
$$\int_{\SU^{d-1}}a(\bb)\sigma(\bb^T\bx)\pi(d\bb)$$
for some $\pi$ and $\|f\|_{\cH_\pi}\leq Q$, which means $f\in\cB_Q$. 
Hence, $\Gamma_Q\subset\cB_Q$. Next, since the training data $\{\bx_i\}_{i=1}^n$ and the initialization are fixed, we have 
\begin{align*}
\cA_t(f,\{\bx_i\}_{i=1}^n,\Theta_0) \in \mbox{span}\left(\ga(\cdot,\bx_1), \dots, \ga(\cdot,\bx_n), f_m(\cdot;\Theta_0)\right).
\end{align*}
Therefore, by Lemma~\ref{lm:Barron}, we obtain
\begin{align*}
\sup_{f\in\cB_Q} \|f-\cA_t(f,\{\bx_i\}_{i=1}^n,\Theta_0) \|_{\rho} &\geq \sup_{f\in\cB_Q}\,\inf_{h\in \mbox{span}(\ga(\cdot,\bx_i),\dots,\ga(\cdot,\bx_n), f_m(\cdot;\Theta_0))} \|f-h\| \nonumber \\
 & \geq  \sup_{f\in\Gamma_Q} \,\inf_{h\in \mbox{span}(\ga(\cdot,\bx_i),\dots,\ga(\cdot,\bx_n), f_m(\cdot;\Theta_0))} \|f-h\| \nonumber \\
 & \geq \frac{\kappa Q}{d(n+1)^{1/d}},
\end{align*}
for some universal constant $\kappa$. 
\end{proof}

\section{Analysis of the general case}
In this section, we will relax the requirement of the network width. We will make the following assumption on the target function.
\begin{assumption}\label{assump:early_stop}
We assume that the target function $f^*$ admits the following integral representation
\begin{equation}
f^*(x)=\int_{\SS^{d-1}} a^*(\bb)\sigma(\bb^Tx)d\pi_0(\bb),
\end{equation}
with $\gamma(f^*)\Def \max\{1,\sup_{\bb\in\SS^{d-1}} |a^*(\bb)|\}< \infty$.
\end{assumption}
 Let $\cH_{k^a}$ be the RKHS  induced by $k^a(\cdot,\cdot)$. It was shown in \cite{rahimi2008uniform} that $ \|f^*\|_{\cH_{k^a}}=\sqrt{\EE_{\pi_0}[|a^*(\bb)|^2]}\leq \gamma(f^*)$. Thus the assumption above implies that $f^*\in \cH_{k^a}$. 

The following approximation result is essentially the same as the ones in  \cite{rahimi2008uniform,rahimi2009weighted}. Since we are interested in the explicit control for the norm of the solution, we provide a complete proof in Appendix~\ref{sec: approx}.
\begin{lemma}\label{lm:a_star}
Assume that the target function $f^*$ satisfies Assumption~\ref{assump:early_stop}. Then for any $\delta>0$, with probability at least $1-\delta$ over the choice of $B_0$, there exists $\ba^*\in\RR^{m}$ such that
\begin{equation}\label{eq:a_star}
\cR(\ba^*,B_0)\leq\frac{\gamma^2(f^*)}{m}\left(1+\sqrt{2\ln(1/\delta)}\right)^2 
\end{equation}
 \begin{equation}
 \|\ba^*\|\leq\frac{\gamma(f^*)}{\sqrt{m}},
 \end{equation}
 where $\cR(\ba^*,B_0)=\|f_m(\cdot;\ba^*,B_0)-f^*(\cdot)\|^2_{\rho}$ is the population risk.
\end{lemma}

The following generalization bound for the random feature model will be used  later.
\begin{lemma}\label{lm:rad}
For fixed $B_0$, and any $\delta>0$, with probability no less than $1-3\delta$ over the choice of the training data, we have 
\begin{equation}
\left| \cR(\ba,B_0)-\hat{\cR}_n(\ba,B_0) \right|\leq \frac{2(2\sqrt{m}\|\ba\|+1)^2}{\sqrt{n}}\left(1+\sqrt{2\ln\left(\frac{2}{\delta}(\|\ba\|+\frac{1}{\|\ba\|})\right)}\ \right)
\end{equation}
for any $\ba \in \RR^m$.
\end{lemma}
Please see Appendix~\ref{sec: rf-rad} for the proof.

\subsection{Optimization results}
We first show that the gradient descent algorithm can reduce the empirical risk to $\cO(\frac{1}{m}+\frac{1}{\sqrt{n}})$.
Here we will assume $m \ge n$.  This assumption is not used in the next subsection, except for Corollary~\ref{col: early-stop}.

\begin{theorem}\label{thm:early_stop_conv}
Take $\beta=\frac{c}{m}$ for some absolute constant $c$.  
Assume that  the target function $f^*$ satisfies Assumption~\ref{assump:early_stop}, and $\|f^*\|_\infty\leq1$. Then, 
for any $\delta \in (0,1)$,  with probability no less than $1-4\delta$ we have
\begin{align*}
\hat{\cR}_n(\ba_t,B_t)\leq C\left(\frac{1}{m}+\frac{1}{mt}+\frac{1}{\sqrt{n}}\right),
\end{align*}
for any $t>0$, where $C$ is a constant depending on $\delta$, $\gamma(f^*)$ and $c$. 
\end{theorem}

The next three lemmas give bounds on the changes of the parameters.
\begin{lemma}\label{lm:param1}
Let $\beta=\frac{c}{m}$, and $T$ be a fixed constant. Then there exists constant $C_T$ depending on $T$, such that for any $0\leq t\leq T$, 
\begin{equation}
\|\ba_t\|\leq C_T \left(\frac{c}{\sqrt{m}}+\sqrt{m}t\right),\ \ \|B_t\|\leq C_T \left(\frac{c t}{\sqrt{m}}+\sqrt{m}\right),
\end{equation}
and
\begin{equation}
\|B_t-B_0\|\leq C_T(c+1)\left( \frac{c}{\sqrt{m}}t+\frac{\sqrt{m}}{2}t^2 \right).
\end{equation}
\end{lemma}

\begin{proof}
By the gradient descent dynamics,  we have
\begin{align*}
\|a_k(t)\| &\leq \|a_k(0)\|+\int_0^t \|\bb_k(s)\|\sqrt{\hat{\cR}_n(\ba_0,B_0)}ds \leq \|a_k(0)\|+(\gamma(f^*)+1)\int_0^t \|\bb_k(s)\|ds \\
\|\bb_k(t)\| &\leq \|\bb_k(0)\|+\int_0^t \|a_k(s)\|\sqrt{\hat{\cR}_n(\ba_0,B_0)}ds \leq \|\bb_k(0)\|+(\gamma(f^*)+1)\int_0^t \|a_k(s)\|ds
\end{align*}
Since $\|a_k(0)\|=\frac{c}{m}$ and $\|\bb_k(0)\|=1$, we have
\begin{align*}
\|a_k(t)\| &\leq \cosh((c+1)t)\frac{c}{m}+\sinh((c+1)t), \\
\|\bb_k(t)\| &\leq \sinh((c+1)t)\frac{c}{m}+\cosh((c+1)t).
\end{align*}
If $t\leq T$, since $\cosh((c+1)t)\leq\frac{e^{(c+1)T}+1}{2}$ and $\sinh((c+1)t)\leq\frac{e^{(c+1)T}+1}{2}t$, we have
\begin{align*}
\|a_k(t)\| &\leq C_T \left(\frac{c}{m}+t\right), \\
\|\bb_k(t)\| &\leq C_T \left(\frac{c t}{m}+1\right),
\end{align*}
with $C_T=\frac{e^{(c+1)T}+1}{2}$. Hence, we have
\begin{equation}
\|\ba_t\|\leq C_T \left(\frac{c}{\sqrt{m}}+\sqrt{m}t\right),\ \ \|B_t\|\leq C_T \left(\frac{c t}{\sqrt{m}}+\sqrt{m}\right).
\end{equation}
For $\|B_t-B_0\|$, consider a more refined estimate
\begin{equation}
\|\bb_k(t)-\bb_k(0)\|\leq \int_0^t \|a_k(s)\|\sqrt{\hat{\cR}_n(\ba_0,B_0)}ds \leq (c+1)\int_0^t \|a_k(s)\|ds.
\end{equation}
Plugging in the above estimate for $a_k$, we obtain
\begin{equation}
\|B_t-B_0\|\leq C_T(c+1)\left( \frac{c}{\sqrt{m}}t+\frac{\sqrt{m}}{2}t^2 \right).
\end{equation}
\end{proof}

\begin{lemma}\label{lm:param2}
Let $\gamma=\gamma(f^*)$, $\beta=\frac{c}{m}$, and assume $\sqrt{m}\geq\gamma$. Then, for any $\delta>0$, with probability no less than $1-4\delta$, we have for any $0\leq t\leq T$, 
\begin{equation}
\|\tilde{\ba}_t\|\leq\tilde{C}_T\left(\frac{1}{\sqrt{m}}+\frac{\sqrt{t}}{\sqrt{m}}+\frac{\sqrt{t}}{n^{1/4}}\right).
\end{equation}
where $\tilde{C}_T$ is a constant.
\end{lemma}

\begin{proof}
For $\|\tilde{\ba}_t\|$, consider the Lyapunov function
\begin{equation}\label{eq:lyap}
J(t)=t\left(\hat{\cR}_n(\tilde{\ba}_t,B_0)-\hat{\cR}_n(\ba^*,B_0)\right)+\frac{1}{2}\|\tilde{\ba}_t-\ba^*\|^2.
\end{equation}
Since $\hat{\cR}_n(\tilde{\ba}_t,B_0)$ is convex with respect to $\tilde{\ba}_t$, we have  $\frac{d}{dt}J(t)\leq0$, which
implies $J(t)\leq J(0)$. Hence we have
\begin{equation}
t(\hat{\cR}_n(\tilde{\ba}_t,B_0)-\hat{\cR}_n(\ba^*,B_0))+\frac{1}{2}\|\tilde{\ba}_t-\ba^*\|^2\leq \frac{1}{2}\|\ba_0-\ba^*\|^2.
\end{equation}
Since $\hat{\cR}_n(\tilde{\ba}_t,B_0)\geq0$, we  obtain
\begin{equation}
\|\tilde{\ba}_t-\ba^*\|^2\leq 2t\hat{\cR}_n(\ba^*,B_0))+\|\ba_0-\ba^*\|^2.
\end{equation}
By Lemma~\ref{lm:a_star} and Lemma~\ref{lm:rad}, when $\sqrt{m}\geq\gamma$, with probability no less than $1-4\delta$, 
we have
\begin{align}
\hat{\cR}_n(\ba^*,B_0) &= \cR(\ba^*,B_0)+ \hat{\cR}_n(\ba^*,B_0)-\cR(\ba^*,B_0) \nonumber \\
  &\leq \frac{\gamma^2}{m}\left(1+\sqrt{2\log(\frac{1}{\delta})}\right)^2+\frac{2(2\sqrt{m}\|\ba^*\|+1)^2}{\sqrt{n}}\left(1+\sqrt{2\ln(\frac{2}{\delta}(\|\ba^*\|+\frac{1}{\|\ba^*\|}))}\right) \nonumber \\
  &\leq 2(2\gamma+1)^2\left(1+\sqrt{2\log(\frac{4\sqrt{m}}{\gamma\delta})}\right)^2\left(\frac{1}{m}+\frac{1}{\sqrt{n}}\right).
\end{align}
Therefore we have
\begin{align}
\|\tilde{\ba}_t\|^2 &\leq 2\|\tilde{\ba}_t-\ba^*\|^2+2\|\ba^*\|^2 \nonumber \\
  &\leq 2\|\ba^*\|^2+2\|\ba_0-\ba^*\|^2+4t\hat{\cR}_n(\ba^*,B_0) \nonumber \\
  &\leq \frac{4\gamma^2+2c^2}{m}+8(2\gamma+1)^2\left(1+\sqrt{2\log(\frac{4\sqrt{m}}{\gamma\delta})}\right)^2\left(\frac{1}{m}+\frac{1}{\sqrt{n}}\right)t.
\end{align}
Let $\tilde{C}=\max\{\sqrt{4\gamma^2+2c^2}, 2\sqrt{2}(2\gamma+1)\left(1+\sqrt{2\log(\frac{4\sqrt{m}}{\gamma\delta})}\right)\}$, 
we get
\begin{equation}
\|\tilde{\ba}_t\|\leq\tilde{C}\left(\frac{1}{\sqrt{m}}+\frac{\sqrt{t}}{\sqrt{m}}+\frac{\sqrt{t}}{n^{1/4}}\right).
\end{equation}
\end{proof}

\begin{lemma}\label{lm:param3}
Under the assumptions of Lemmas~\ref{lm:param1} and~\ref{lm:param2},  for any $0\leq t\leq T$,  we have
\begin{equation}
\|\ba_t-\tilde{\ba}_t\|\leq C_T\frac{t^2}{m}(1+mt)(t+m)\left(\frac{1+\sqrt{t}}{\sqrt{m}}+\frac{\sqrt{t}}{n^{1/4}}\right).
\end{equation}
for some constant $C_T$.
\end{lemma}

\begin{proof}
Note that
\begin{align}
\frac{d}{dt}(\ba_t-\tilde{\ba}_t) &= -\frac{1}{n}\sum\limits_{i=1}^n\big(e_i(t)\sigma(B_t\bx_i)-\tilde{e}_i(t)\sigma(B_0\bx_i)\big) \nonumber \\
  &=-\frac{1}{n}\sum\limits_{i=1}^n\ba_t^T\sigma(B_t\bx_i)\sigma(B_t\bx_i)+\frac{1}{n}\sum\limits_{i=1}^n\tilde{\ba}_t^T\sigma(B_0\bx_i)\sigma(B_0\bx_i) \nonumber \\
  & \qquad +\frac{1}{n}\sum\limits_{i=1}^n f^*(\bx)^T(\sigma(B_t\bx_i)-\sigma(B_0\bx_i)) \nonumber \\
  &= -\frac{1}{n}\sum\limits_{i=1}^n\sigma(B_t\bx_i)\sigma(B_t\bx_i)^T(\ba_t-\tilde{\ba}_t)+\frac{1}{n}\sum\limits_{i=1}^n (\sigma(B_t^Tx_i)\sigma(B_t\bx_i)^T-\sigma(B_0\bx_i)\sigma(B_0\bx_i)^T)\tilde{\ba}_t \nonumber \\
  &\qquad +\frac{1}{n}\sum\limits_{i=1}^n f^*(\bx)^T(\sigma(B_t\bx_i)-\sigma(B_0\bx_i)). \label{eq:early_stop_2}
\end{align}
Multiplying $\ba_t-\tilde{\ba}_t$ on both sides of~\eqref{eq:early_stop_2}, 
we get
\begin{align}
\frac{d}{dt}\|\ba_t-\tilde{\ba}_t\|^2 &\leq (\ba_t-\tilde{\ba}_t)^T \frac{2}{n}\sum\limits_{i=1}^n (\sigma(B_t\bx_i)\sigma(B_t\bx_i)^T-\sigma(B_0\bx_i)\sigma(B_0\bx_i)^T)\tilde{\ba}_t \nonumber \\
  & \qquad+(\ba_t-\tilde{\ba}_t)^T \frac{2}{n}\sum\limits_{i=1}^n f^*(\bx)^T(\sigma(B_t\bx_i)-\sigma(B_0\bx_i)) \nonumber \\
  &\leq  2 \|B_t-B_0\|(\|B_t\|\|\tilde{\ba}_t\|+\|B_0\|\|\tilde{\ba}_t\|+1)\|\ba_t-\tilde{\ba}_t\|.
\end{align}
Using the estimates  in Lemmas~\ref{lm:param1} and~\ref{lm:param2}, we obtain 
\begin{equation}
\|\ba_t-\tilde{\ba}_t\|\leq3C_T^2\tilde{C}_T(1+c)^3\frac{t^2}{m}(1+mt)(t+m)\left(\frac{1+\sqrt{t}}{\sqrt{m}}+\frac{\sqrt{t}}{n^{1/4}}\right).
\end{equation}
\end{proof}

\subsubsection*{Proof of Theorem \ref{thm:early_stop_conv}}
Let $\hat{\rho}=\frac{1}{n}\sum_{i=1}^n\delta_{\bx_i}$, 
then we have
\begin{align}
\hat{\cR}_n(\ba_t,B_t) &= \|f(\cdot;\ba_t,B_t)-f^*(\cdot)\|^2_{\hat{\rho}} \nonumber \\
  &\leq 3\left( \|f(\cdot;\ba_t,B_t)-f(\cdot;\tilde{\ba}_t,B_t)\|^2_{\hat{\rho}}+\|f(\cdot;\tilde{\ba}_t,B_t)-f(\cdot;\tilde{\ba}_t,B_0)\|^2_{\hat{\rho}}\right. \nonumber \\
  & \qquad \left.+\hat{\cR}_n(\tilde{\ba}_t,B_0) \right). \label{eq:decomp}
\end{align}
By Cauchy-Schwartz, we have
\begin{equation}\label{eq:cauchy1}
\|f(\bx;A_t,B_t)-f(\bx;\tilde{\ba}_t,B_t)\|^2_{\hat{\rho}}\leq \|\ba_t-\tilde{\ba}_t\|^2\|B_t\|^2,
\end{equation}
\begin{equation}\label{eq:cauchy2}
\|f(\bx;\tilde{\ba}_t,B_t)-f(\bx;\tilde{\ba}_t,B_0)\|^2_{\hat{\rho}}\leq \|\tilde{\ba}_t\|^2\|B_t-B_0\|^2.
\end{equation}

For $\hat{\cR}_n(\tilde{\ba}_t,B_0)$, from Lemma~\ref{lm:a_star}, with probability $ 1- \delta$, there exists $\ba^*$ 
that satisfies~\eqref{eq:a_star}. Thus we have
\begin{align}
\hat{\cR}_n(\tilde{\ba}_t,B_0) &= \left(\hat{\cR}_n(\tilde{\ba}_t,B_0)-\hat{\cR}_n(\ba^*,B_0)\right)+\left(\hat{\cR}_n(\ba^*,B_0)-\cR(\ba^*,B_0)\right)\nonumber \\
  &=: I_1 + I_2. \label{eq:bound_gap1}
\end{align}
By Lemma~\ref{lm:rad}, we can bound $I_2$ as follows,
\begin{equation}
I_2 \leq \frac{2(2\sqrt{m}\|\ba^*\|+1)^2}{\sqrt{n}}\left(1+\sqrt{2\ln(\frac{2}{\delta}(\|\ba^*\|+\frac{1}{\|\ba^*\|}))}\right).
\end{equation}
For $I_1$, consider the Lyapunov function
\begin{equation}
J(t)=t\left(\hat{\cR}_n(\tilde{\ba}_t,B_0)-\hat{\cR}_n(\ba^*,B_0)\right)+\frac{1}{2}\|\tilde{\ba}_t-\ba^*\|^2.
\end{equation}
Since $\hat{\cR}_n(\tilde{\ba}_t,B_0)$ is convex with respect to $\tilde{\ba}_t$, we have  $\frac{d}{dt}J(t)\leq0$, which implies
$J(t)\leq J(0)$. Hence we have
\begin{equation}
\hat{\cR}_n(\tilde{\ba}_t,B_0)\leq \hat{\cR}_n(\ba^*,B_0)+\frac{\|\ba_0-\ba^*\|^2}{2t}.
\end{equation}

Combining all the estimates above, we conclude that for any $\delta>0$, with probability larger than $1-4\delta$, we have
\begin{align}\label{eq:est_0}
\hat{\cR}_n(\ba_t,B_t) &\leq 3\|\ba_t-\tilde{\ba}_t\|^2\|B_t\|^2+3\|\tilde{\ba}_t\|^2\|B_t-B_0\|^2 \nonumber \\
 & \qquad +\frac{6(2\sqrt{m}\|\ba^*\|+1)^2}{\sqrt{n}}\left(1+\sqrt{2\ln(\frac{2}{\delta}(\|\ba^*\|+\frac{1}{\|\ba^*\|}))}\right) \nonumber \\
 &\qquad  +\frac{\gamma^2}{m}\left(1+\sqrt{2\ln(\frac{1}{\delta})}\right)^2+\frac{\|\ba_0-\ba^*\|^2}{2t}. 
\end{align}

For the estimate on $\ba^*$, by Lemma~\ref{lm:a_star}, we have $\|\ba^*\|\leq\frac{\gamma}{\sqrt{m}}$. 
To bound $\|\ba_0-\ba^*\|$,  we have
\begin{equation}
\|\ba_0-\ba^*\|\leq \|\ba_0\|+\|\ba^*\|\leq \frac{c+\gamma}{\sqrt{m}}.
\end{equation}
Together with the estimates in Lemmas~\ref{lm:param1}, \ref{lm:param2} and~\ref{lm:param3}, and without loss of generality assuming that $\gamma\geq1$, we obtain
\begin{align}\label{eq:est_final}
\hat{\cR}_n(\ba_t,B_t) &\leq C\left( \frac{1}{m}+\frac{1}{mt}+\frac{1}{\sqrt{n}}\left(1+\sqrt{t}+\frac{\sqrt{mt}}{n^{1/4}}\right)^2 \right. \nonumber \\
  &\qquad \left. +\frac{t^2}{m^2}(1+mt)^2\left(1+\frac{t^2}{m^2}(t+m)^4\right)\left(1+\sqrt{t}+\frac{\sqrt{mt}}{n^{1/4}}\right)^2\right).
\end{align}
for $t\in[0,T]$, and some constant $C$ (we can choose $C=27C_T^6\tilde{C}^2_T(c+1)^8(2\gamma+1)^2$). 

If we assume $m\geq n$ and take $t\in[0,\frac{\sqrt{n}}{m}]$, then we can take $T=1$ and obtain 
\begin{equation}\label{eq:slow_conv1}
\hat{\cR}(\ba_t,B_t)\leq C\left(\frac{1}{m}+\frac{1}{mt}+\frac{1}{\sqrt{n}}\right) \quad \forall\,\, 0\leq t \leq 1,
\end{equation} 
for some constant $C$. Moreover, since $\hat{\cR}_n(\ba_t,B_t)$ is non-increasing,  $\hat{\cR}_n(\ba_t,B_t)\leq\hat{\cR}_n(\ba_{\sqrt{n}/m},B_{\sqrt{n}/m})$.  Hence for any $t>\frac{\sqrt{n}}{m}$, we  have
\begin{equation}\label{eq:slow_conv2}
\hat{\cR}_n(\ba_t,B_t)\leq C\left(\frac{1}{m}+\frac{2}{\sqrt{n}}\right),
\end{equation} 
for some constant $C$. Combining~\eqref{eq:slow_conv1} and~\eqref{eq:slow_conv2},  we complete the proof for all $t$.

\subsection{Generalization results}
{ The following theorem provides an upper bound for the population error of GD solutions at any time $t\in[0,\infty)$. It tells that one can use early stopping  to reach the optimal error in the absence of over-parametrization.}
\begin{theorem}\label{thm:early_stop}
Take $\beta=\frac{c}{m}$ for some constant $c$. Assume that the target function $f^*$ satisfies Assumption~\ref{assump:early_stop}, and $\|f^*\|_\infty\leq1$.   Fix any positive constant $T$.  Then for 
$\delta>0$, with probability no less than $1-4\delta$ we have, for $t \le T$
\begin{align}
\cR(\ba_t,B_t) &\leq C\left( \frac{1}{m}+\frac{1}{mt}+\frac{1}{\sqrt{n}}\left(1+\sqrt{t}+\frac{\sqrt{mt}}{n^{1/4}}\right)^2 \right. \nonumber \\
  & \left. +\frac{t^2}{m^2}(1+mt)^2\left(1+\frac{t^2}{m^2}(t+m)^4\right)\left(1+\sqrt{t}+\frac{\sqrt{mt}}{n^{1/4}}\right)^2\right). \label{eq:early_stop_gen}
\end{align}
where $C$ is a constant depending only on $T$, $\delta$, $\gamma(f^*)$ and $c$. 
\end{theorem}

As a consequence, we  have the following early-stopping results.
\begin{corollary}[Early-stopping solution]\label{col: early-stop}
Assume that $m>n$. Let $t=\frac{\sqrt{n}}{m}$.
Under the condition of Theorem~\ref{thm:early_stop}, we have
\begin{equation}
\cR(\ba_t,B_t)\lesssim \frac{1}{m}+\frac{1}{\sqrt{n}}.
\end{equation}
\end{corollary}
\begin{remark}
From these results we conclude that for target functions in a certain RKHS, with high probability the gradient descent dynamics can find a solution with good generalization properties in a short time. Compared to the long-term analysis in the last section, this theorem does not require $m$ to be very large. It  works in the ``{mildly} over-parameterized'' regime. 
\end{remark}
The following Corollary provides a more detailed study of the balance between $m$, $n$ and $t$ to achieve best rates for $\cR(\ba_t,B_t)$.
\begin{corollary}
Assume $m=n^p$ for some $p\geq0$. Then, if $p\leq\frac{7}{8}$, take $t=n^{-\frac{3p}{7}}$, we have
\begin{equation}
\cR(\ba_t,B_t)\lesssim n^{-\frac{4}{7}p}.
\end{equation}
If $p>\frac{7}{8}$, take $t=n^{-p+\frac{1}{2}}$, we have
\begin{equation}
\cR(\ba_t,B_t)\lesssim n^{-\frac{1}{2}}.
\end{equation}
\end{corollary}

\begin{proof}
Let $m=n^p$ and $t=n^r$. We assume $r\leq0$, then 
\begin{equation}
\left(1+\sqrt{t}+\frac{\sqrt{mt}}{n^{1/4}}\right)^2 \lesssim 1+\frac{mt}{\sqrt{n}}. 
\end{equation}
Expand the right hand side of~\eqref{eq:early_stop_gen}, we obtain
\begin{eqnarray}
\cR(\ba_t,B_t) &\lesssim& n^{-p}+n^{-r-p}+n^{-\frac{1}{2}}+n^{r+p-1}+n^{2r-2p}+n^{3r-p-\frac{1}{2}} \nonumber \\
  && +n^{4r}+n^{5r+p-\frac{1}{2}}+n^{6r+2p}+n^{7r+3p-\frac{1}{2}}. \label{eq:expand}
\end{eqnarray}
For each $p\geq0$, we are going to find the corresponding $r$ 
for which the maximum value among all the terms at the right hand side of~\eqref{eq:expand} is minimized. When $r=-p$, we have $-r-p=0$.
Thus the second  term is larger than any other terms. Hence, we only have to consider the case when $-p\leq r\leq0$. In this interval, we only
need to compare the terms with powers $-r-p$, $r+p-1$, $6r+2p$ and $7r+3p-\frac{1}{2}$ and neglect all other terms. 
The desired results are then obtained  by comparing the second term with the other three terms. 
\end{proof}

Now we prove Theorem~\ref{thm:early_stop}.
\begin{proof}
Similar to~\eqref{eq:decomp}, we have
\begin{align}
\cR(\ba_t,B_t) &= \|f(\bx;\ba_t,B_t)-f^*(\bx)\|^2_{\rho} \nonumber \\
  &\leq 3\left( \|f(\bx;\ba_t,B_t)-f(\bx;\tilde{\ba}_t,B_t)\|^2_{\rho}+\|f(\bx;\tilde{\ba}_t,B_t)-f(\bx;\tilde{\ba}_t,B_0)\|^2_{\rho}\right. \nonumber \\
  & \left.+\mathcal{R}(\tilde{\ba}_t,B_0) \right). \label{eq:decomp2}
\end{align}
Here $\rho$ is the distribution of input data $\bx$. 
For the first two terms in~\eqref{eq:decomp2}, we have the same estimates as in~\eqref{eq:cauchy1} and~\eqref{eq:cauchy2}.
For $\cR(\tilde{\ba}_t,B_0)$, we have
\begin{align}
\cR(\tilde{\ba}_t,B_0) &= \left(\cR(\tilde{\ba}_t,B_0)-\hat{\cR}_n(\tilde{\ba}_t,B_0)\right)+\left(\hat{\cR}_n(\tilde{\ba}_t,B_0)-\hat{\cR}_n(\ba^*,B_0)\right) \nonumber \\
  &\qquad\quad +\left(\hat{\cR}_n(\ba^*,B_0)-\cR(\ba^*,B_0)\right). \label{eq:bound_tilde2}
\end{align}
The right hand side of~\eqref{eq:bound_tilde2} has one more term than~\eqref{eq:bound_gap1}, and 
additional term can be bounded as 
\begin{equation}
\cR(\tilde{\ba}_t,B_0)-\hat{\cR}_n(\tilde{\ba}_t,B_0)\leq \frac{2(2\sqrt{m}\|\tilde{\ba}_t\|+1)^2}{\sqrt{n}}\left(1+\sqrt{2\ln(\frac{2}{\delta}(\|\tilde{\ba}_t\|+\frac{1}{\|\tilde{\ba}_t\|}))}\right).
\end{equation}
Hence, for any $\delta>0$, with probability larger than $1-4\delta$, we have
\begin{align}\label{eq:est_1}
\cR(\ba_t,B_t) &\leq 3\|\ba_t-\tilde{\ba}_t\|^2\|B_t\|^2+3\|\tilde{\ba}_t\|^2\|B_t-B_0\|^2 \nonumber \\
 &\qquad +\frac{6(2\sqrt{m}\|\tilde{\ba}_t\|+1)^2}{\sqrt{n}}\left(1+\sqrt{2\ln(\frac{2}{\delta}(\|\tilde{\ba}_t\|+\frac{1}{\|\tilde{\ba}_t\|}))}\right) \nonumber \\
 &\qquad +\frac{6(2\sqrt{m}\|\ba^*\|+1)^2}{\sqrt{n}}\left(1+\sqrt{2\ln(\frac{2}{\delta}(\|\ba^*\|+\frac{1}{\|\ba^*\|}))}\right) \nonumber \\
 &\qquad +\frac{\gamma^2}{m}\left(1+\sqrt{2\ln(\frac{1}{\delta})}\right)^2+\frac{\|\ba_0-\ba^*\|^2}{2t}. 
\end{align}

Using the estimates of $\|\ba_t-\tilde{\ba}_t\|$, $\|B_t\|$, $\|\tilde{\ba}_t\|$, $\|B_t-B_0\|$, $\|\ba^*\|$ and $\|\ba_0-\ba^*\|$  derived in the previous lemmas, and assuming that $\frac{1+\sqrt{t}}{\sqrt{m}}+\frac{\sqrt{t}}{n^{1/4}}\leq1$, we obtain
\begin{align}\label{eq:est_final2}
\cR(\ba_t,B_t) &\leq C\left( \frac{1}{m}+\frac{1}{mt}+\frac{1}{\sqrt{n}}\left(1+\sqrt{t}+\frac{\sqrt{mt}}{n^{1/4}}\right)^2 \right. \nonumber \\
  &\qquad \left. +\frac{t^2}{m^2}(1+mt)^2\left(1+\frac{t^2}{m^2}(t+m)^4\right)\left(1+\sqrt{t}+\frac{\sqrt{mt}}{n^{1/4}}\right)^2\right).
\end{align}
In~\eqref{eq:est_final2}, the constant $C$ can be chosen as $C=27C_T^6\tilde{C}^2(c+1)^8(2\gamma+1)^2$.

\end{proof}

\section{Numerical experiments}

In this section, we present some numerical results to illustrate our theoretical analysis.

\subsection{Fitting random labels}
The first experiment studies the convergence of GD dynamics for over-parametrized two-layer neural networks with different initializations. We uniformly sample $\{\bx_i\}_{i=1}^{n}$ from $\SS^{d-1}$, and for each $\bx_i$ we specify a label $y_i$, which is uniformly drawn from $[-1,1]$. In the experiments, we choose $n=50,d=50$, and network width $m=10,000\gg n$.  Six initializations of different magnitudes are tested.  Figure~\ref{fig:conv} shows the training curves.
 
 We see that the GD algorithm for the neural network models converges exponentially fast for  all  initializations considered, even for the case when $\beta=m$.  This is consistent with the results of Theorem~\ref{thm: optimization}. 

 \begin{figure}[!h]
\centering
\includegraphics[width=0.4\textwidth]{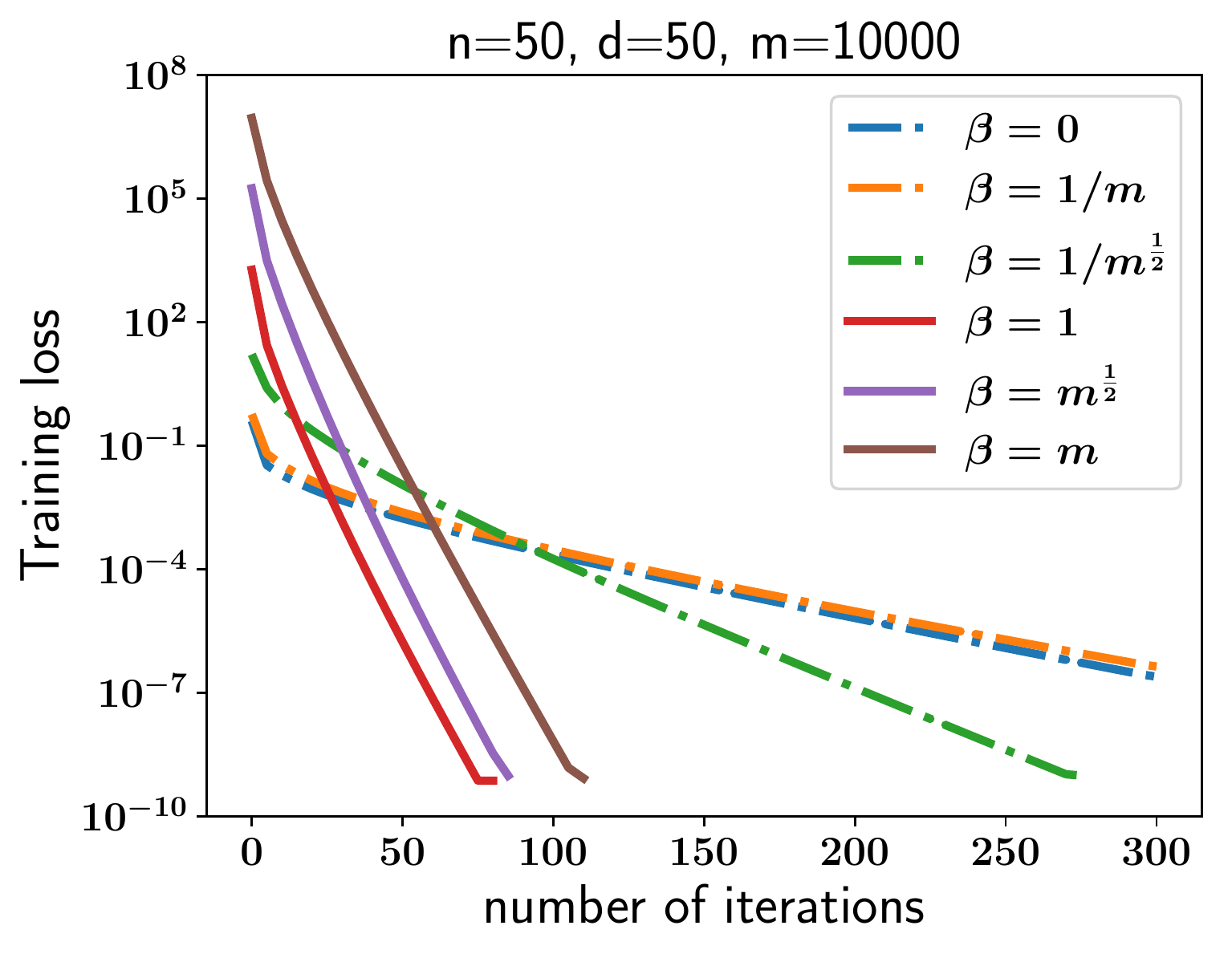}
\caption{Convergence of the GD algorithm for over-parameterized two-layer neural network models on randomly labeled data.  Here $\beta$ denotes the magnitude of the initialization of $\ba$. Different curves correspond to different initializations.  \label{fig:conv}
}
\end{figure}

\subsection{Learning the one-neuron function}
The next experiment compares the GD dynamics of two-layer neural networks and random feature models. We consider the target function $f^*(\bx):= \sigma(\be_1^T\bx)$ with $\be_1=(1,0,\cdots,0)^T\in\RR^{d}$. The training set is given by $\{(\bx_i,f^*(\bx_i))\}_{i=1}^n$, with $\{\bx_i\}_{i=1}^n$ independently drawn from $\SS^{d-1}$.

We first choose $n=50,d=10$ to build the training set, and then use the gradient descent algorithm
with learning rate $\eta=0.01$ to train two-layer neural network and random feature models. We initialize the models  using $\beta=0$. In addition, $10^4$ new samples are drawn to evaluate the test error. Figure~\ref{fig:rf_compare} shows the training and test error curves of the two models of three widths: $m=4, 50, 1000$. We  see that, when the width is very small, the GD algorithm for the random feature model  does not converge, while it does converge for the neural network model and the resulting model does generalize. This is likely due to the special target function we have chosen here.
For the intermediate width ($m=50$), the GD algorithm for both models converges, and it converges faster for the neural network model  than  for the random feature model.
The test accuracy  is slightly better for the resulting neural network model (but  not as good as for the case when $m=4$).
 When $m=1000$, the behavior of the GD algorithm for two models is  almost the same.
\begin{figure}
\centering
\includegraphics[width=0.31\textwidth]{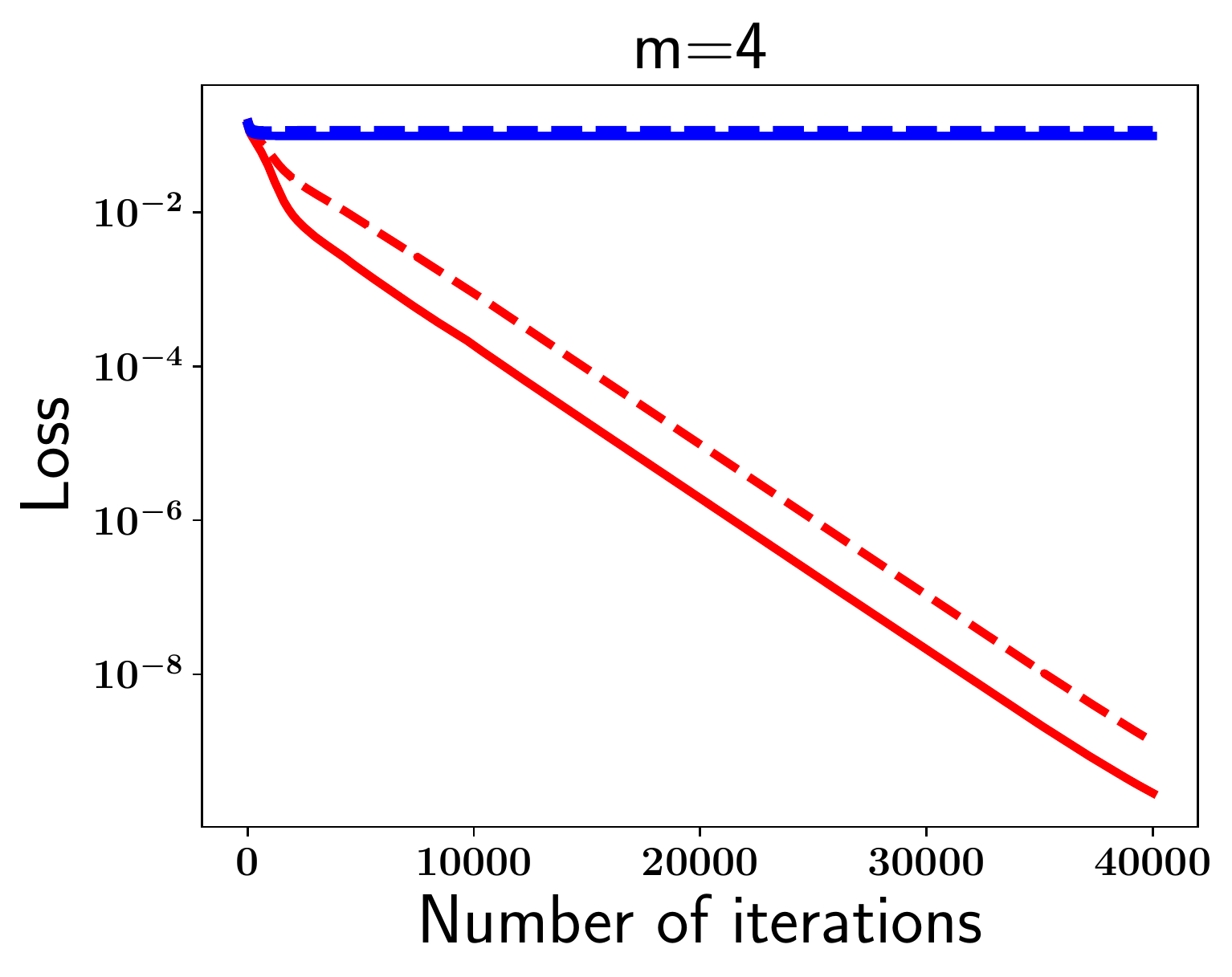}
\includegraphics[width=0.31\textwidth]{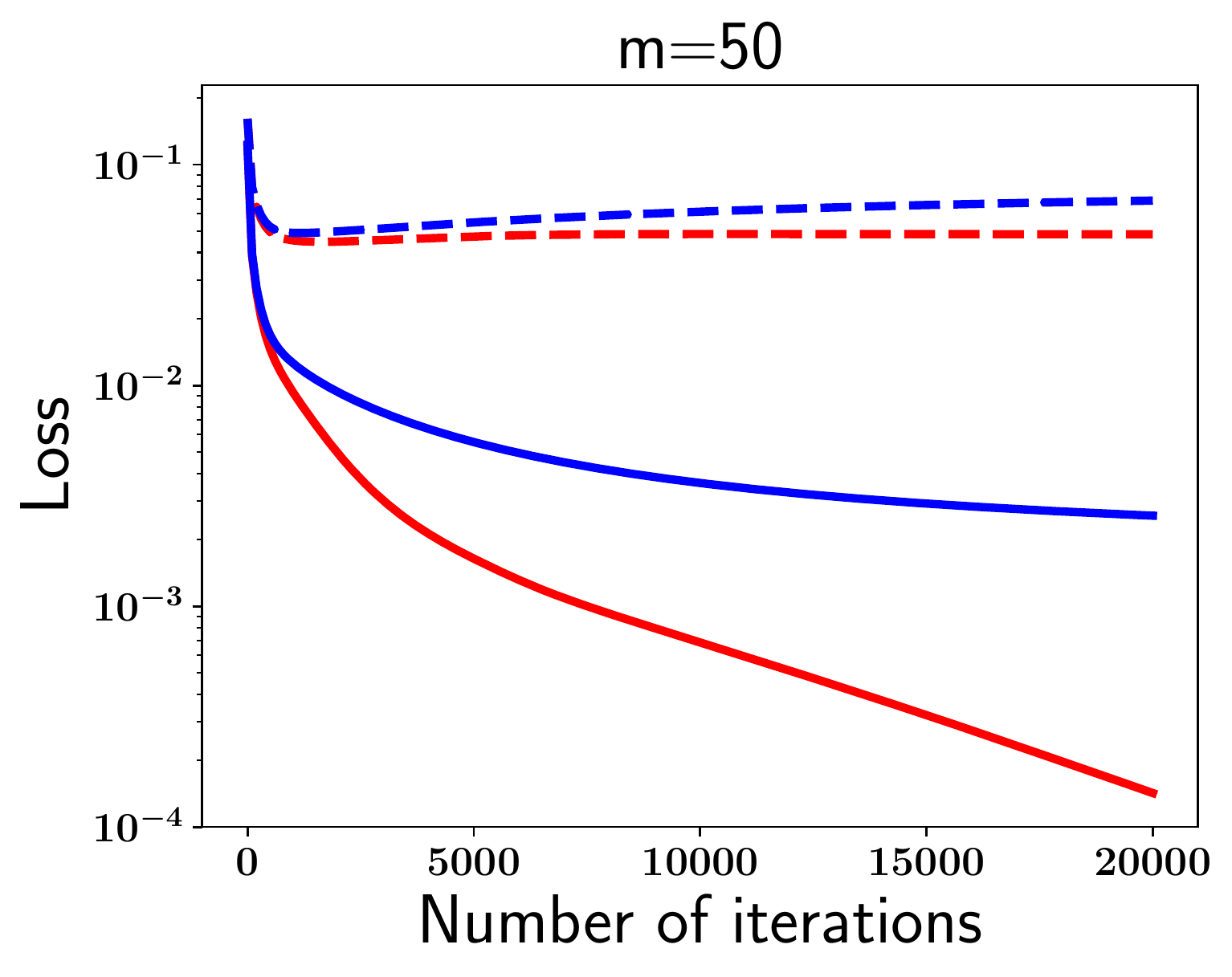}
\includegraphics[width=0.31\textwidth]{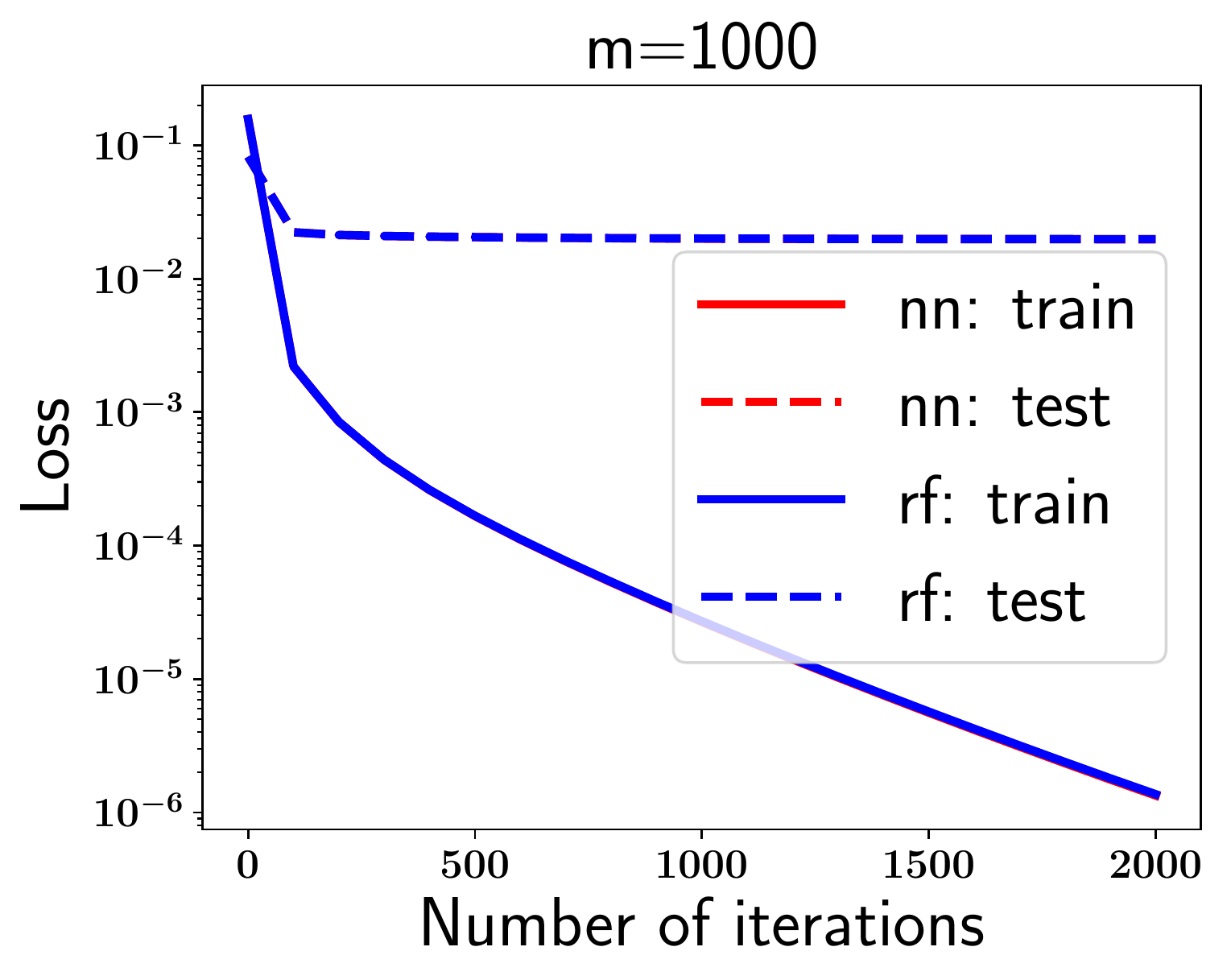}
\caption{Training and testing losses for the neural network and random feature models using the GD algorithm, starting from zero initialization of $\ba$. Left: $m=4$; Middle: $m=50$; Right: $m=1000$.
\label{fig:rf_compare}
}
\end{figure}

Finally, we study  the generalization properties of neural network models of different width. 
We train two-layer neural networks of different width until the training error is below $10^{-5}$. Then we measure the test error. We compare the test error with that of the regularized model proposed in~\cite{ma2018priori}:
 \begin{equation}\label{eqn: path-reg}
\text{minimize}_{\Theta}\,\, \hat{\cR}_n(\Theta) +  \lambda\sqrt{\frac{\ln(d)}{n}} \|\Theta\|_{\cP},
\end{equation}
where 
\[
 \|\Theta\|_{\cP}= \sum_{k=1}^m |a_k|\|\bb_k\|.
\]
 The results are showed in Figure~\ref{fig:reg_compare}.
 One sees that  when the width is small, the test error is small for both methods. 
 However, when the width becomes very large, the un-regularized neural network model does not generalized well. 
 In other words,  implicit regularization fails.

 {
The above results  are consistent with the theoretical lower bound \eqref{eqn: implicit}, which states that learning with GD suffers from the curse of dimensionality for functions in Barron space. Here the one-neuron function serves as a specific example. Intuitively, the one-neuron target function $f^*(x) = \sigma((\bw^*)^T\bx)$ only relies on the specific direction $\bw^*$.  However the basis $\{\sigma(\bw^T\bx)\}_{j=1}^m$ are uniformly drawn from $\SS^{d-1}$. In high dimension, we know $\langle \bw_j,\bw^*\rangle \approx 0$ for any $\bw_j$ uniformly drawn from $\SS^{d-1}$. Therefore, it is not surprising to see that learning with uniform features suffers from the curse of dimensionality. 
}
 
 \begin{figure}
\centering
\includegraphics[width=0.4\textwidth]{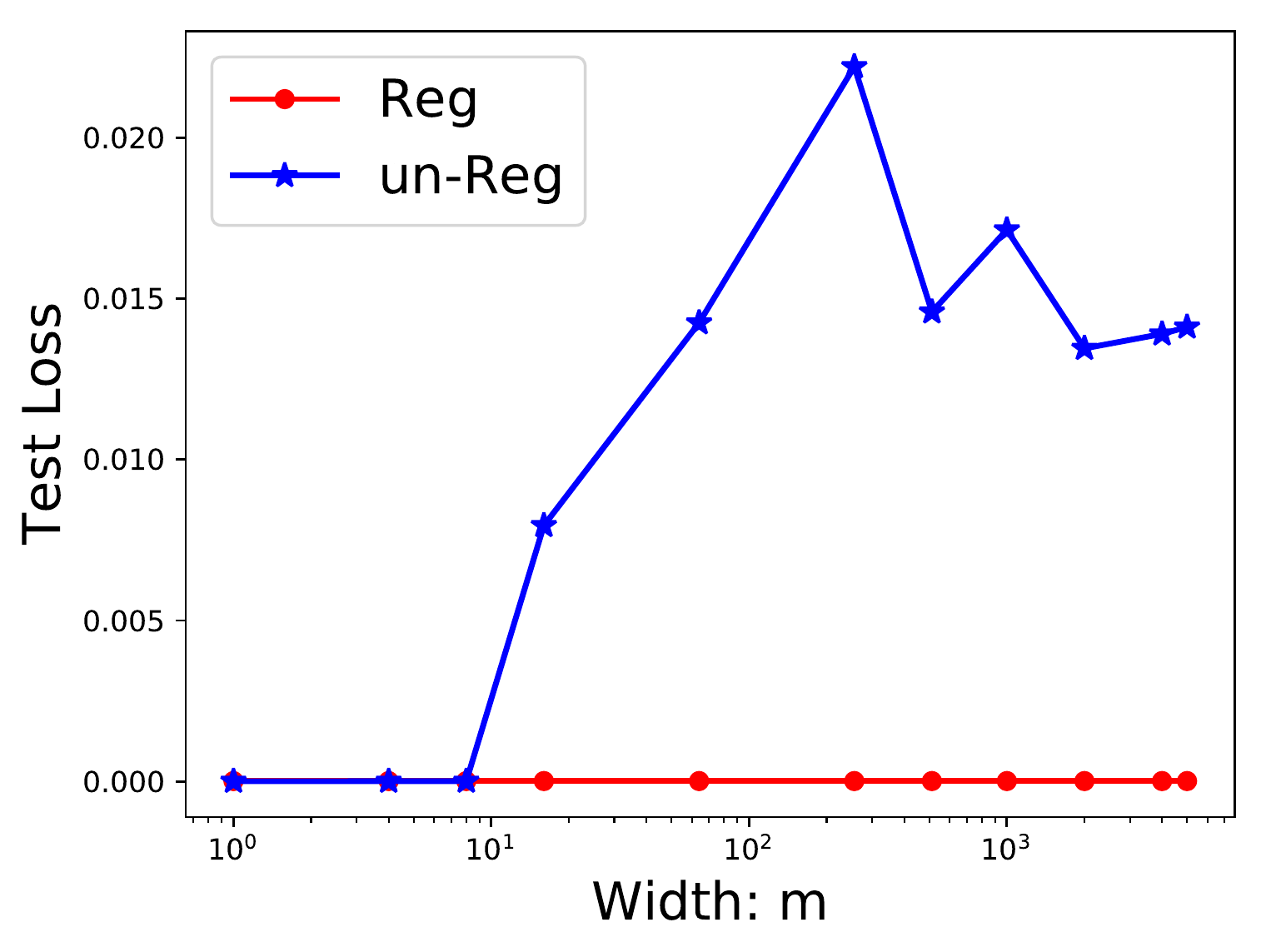}
\caption{Testing errors of two-layer neural network models with different widths, compared with the regularized neural network model. For the regularized model, we choose $\lambda=0.01$. }\label{fig:reg_compare}
\end{figure}

\section{Conclusion}

To put things into perspective, let us first recall some results from \cite{ma2018priori}.
\begin{enumerate}
\item One can define  a space of functions called the Barron space.
The Barron space is the union of all RKHS with kernels defined by
$$k(\bx, \bx') = \EE_{\bb \sim \pi} [\sigma(\bb^T \bx) \sigma(\bb^T \bx')]
$$
with respect to all probability distributions $\pi$.

\item For regularized models with a suitably crafted regularization term, optimal generalization error estimates (i.e. rates that are comparable to the Monte Carlo rates) can be established for  all  target functions in the Barron space.
\end{enumerate}

In the present paper, we have shown that for over-parametrized
two-layer neural networks without explicit regularization,  the gradient descent algorithm is sufficient for the purpose of
optimization. But to obtain dimension-independent error rates for generalization, one has to require
that the target function be in the RKHS with a kernel defined by the initialization.
In other words, given a target function in the Barron space, in order for implicit regularization to work,
one has to know beforehand the kernel function  for that target function and use that kernel function
to initialize the GD algorithm. This requirement is certainly impractical.
In the absence of such a knowledge, one should expect to encounter the curse of dimensionality for
general target functions in Barron space, as is proved in this paper.

We have also studied the case with general network width. Our results point to the same direction as 
for the over-parametrized regime although in the general case, one has to rely on early stopping to 
obtain good generalization error bounds.
Our analysis does not rule out completely the possibility that in some scaling regimes of $n, m, t$, the GD algorithm for two-layer neural network models may have better generalization properties than that of the related kernel method.

Our analysis was carried out  under special circumstances, e.g. with a particular choice of $\pi_0$ and 
a very special  domain $\SU^{d-1}$ for the input.  While it is certainly possible to extend this analysis to more general settings, we feel
that the value of such an analysis is limited since our main message is a negative one: Without explicit
regularization, the generalization properties of two-layer neural networks are likely to be no better
than that of the kernel method.

From a technical viewpoint, our analysis was facilitated greatly by the fact that 
the dynamics of the $\bb$'s is much slower than that of the $\ba$'s,
as  a consequence of the smallness of $\beta$.
As a result, the $\bb$'s are effectively frozen in the GD dynamics. 
While this is the same setup as the ones used in practice, one can also 
imagine putting out an explicit scaling factor to account for the smallness of $\beta$, e.g.
\begin{equation}
f_m(\bx, \Theta) = \frac{1}{m} \sum_{k=1}^m a_k \sigma(\bb_k^T \bx)
\end{equation}
as in  \cite{song2018mean,rotskoff2018parameters,sirignano2018mean}.
In this case, the separation of time scales is no longer valid and one can potentially obtain
a very different picture. While this is certainly an interesting avenue to pursue, so far there are
no results concerning the effect of implicit regularization in such a setting.

{\bf Acknowledgement:}
The work presented here is supported in part by a gift to Princeton University from iFlytek
and the ONR grant N00014-13-1-0338.

\bibliographystyle{plain}
\bibliography{dl_ref}

\begin{thebibliography}{10}

\bibitem{allen2018learning}
Zeyuan Allen-Zhu, Yuanzhi Li, and Yingyu Liang.
\newblock Learning and generalization in overparameterized neural networks,
  going beyond two layers.
\newblock {\em arXiv preprint arXiv:1811.04918}, 2018.

\bibitem{allen2018convergence}
Zeyuan Allen-Zhu, Yuanzhi Li, and Zhao Song.
\newblock A convergence theory for deep learning via over-parameterization.
\newblock {\em arXiv preprint arXiv:1811.03962}, 2018.

\bibitem{aronszajn1950theory}
Nachman Aronszajn.
\newblock Theory of reproducing kernels.
\newblock {\em Transactions of the American mathematical society},
  68(3):337--404, 1950.

\bibitem{arora2019fine}
Sanjeev Arora, Simon~S Du, Wei Hu, Zhiyuan Li, and Ruosong Wang.
\newblock Fine-grained analysis of optimization and generalization for
  overparameterized two-layer neural networks.
\newblock {\em arXiv preprint arXiv:1901.08584}, 2019.

\bibitem{barron1993universal}
Andrew~R. Barron.
\newblock Universal approximation bounds for superpositions of a sigmoidal
  function.
\newblock {\em IEEE Transactions on Information theory}, 39(3):930--945, 1993.

\bibitem{braun2006accurate}
Mikio~L Braun.
\newblock Accurate error bounds for the eigenvalues of the kernel matrix.
\newblock {\em Journal of Machine Learning Research}, 7(Nov):2303--2328, 2006.

\bibitem{breiman1993hinging}
Leo Breiman.
\newblock Hinging hyperplanes for regression, classification, and function
  approximation.
\newblock {\em IEEE Transactions on Information Theory}, 39(3):999--1013, 1993.

\bibitem{cao2019generalization}
Yuan Cao and Quanquan Gu.
\newblock A generalization theory of gradient descent for learning
  over-parameterized deep {ReLU} networks.
\newblock {\em arXiv preprint arXiv:1902.01384}, 2019.

\bibitem{chizat2018note}
Lenaic Chizat and Francis Bach.
\newblock A note on lazy training in supervised differentiable programming.
\newblock {\em arXiv preprint arXiv:1812.07956}, 2018.

\bibitem{daniely2017sgd}
Amit Daniely.
\newblock {SGD} learns the conjugate kernel class of the network.
\newblock In {\em Advances in Neural Information Processing Systems}, pages
  2422--2430, 2017.

\bibitem{du2018deepgradient}
Simon~S Du, Jason~D Lee, Haochuan Li, Liwei Wang, and Xiyu Zhai.
\newblock Gradient descent finds global minima of deep neural networks.
\newblock {\em arXiv preprint arXiv:1811.03804}, 2018.

\bibitem{du2018gradient}
Simon~S. Du, Xiyu Zhai, Barnabas Poczos, and Aarti Singh.
\newblock Gradient descent provably optimizes over-parameterized neural
  networks.
\newblock In {\em International Conference on Learning Representations}, 2019.

\bibitem{ma2018priori}
Weinan E, Chao Ma, and Lei Wu.
\newblock A priori estimates for two-layer neural networks.
\newblock {\em arXiv preprint arXiv:1810.06397}, 2018.

\bibitem{jacot2018neural}
Arthur Jacot, Franck Gabriel, and Cl{\'e}ment Hongler.
\newblock Neural tangent kernel: Convergence and generalization in neural
  networks.
\newblock In {\em Advances in neural information processing systems}, pages
  8580--8589, 2018.

\bibitem{kawaguchi2016deep}
Kenji Kawaguchi.
\newblock Deep learning without poor local minima.
\newblock In {\em Advances in neural information processing systems}, pages
  586--594, 2016.

\bibitem{keskar2016large}
Nitish~S. Keskar, Dheevatsa Mudigere, Jorge Nocedal, Mikhail Smelyanskiy, and
  Ping~T.P. Tang.
\newblock On large-batch training for deep learning: Generalization gap and
  sharp minima.
\newblock In {\em In International Conference on Learning Representations
  (ICLR)}, 2017.

\bibitem{klusowski2016risk}
Jason~M Klusowski and Andrew~R Barron.
\newblock Risk bounds for high-dimensional ridge function combinations
  including neural networks.
\newblock {\em arXiv preprint arXiv:1607.01434}, 2016.

\bibitem{krizhevsky2012a}
Alex Krizhevsky, Ilya Sutskever, and Geoffrey~E. Hinton.
\newblock Imagenet classification with deep convolutional neural networks.
\newblock In {\em Advances in neural information processing systems}, pages
  1097--1105, 2012.

\bibitem{lecun2015deep}
Y.~LeCun, Y.~Bengio, and G.~Hinton.
\newblock Deep learning.
\newblock {\em Nature}, 521(7553):436--444, 2015.

\bibitem{li2018learning}
Yuanzhi Li and Yingyu Liang.
\newblock Learning overparameterized neural networks via stochastic gradient
  descent on structured data.
\newblock In {\em Advances in Neural Information Processing Systems}, 2018.



\bibitem{neyshabur2014search}
Behnam Neyshabur, Ryota Tomioka, and Nathan Srebro.
\newblock In search of the real inductive bias: On the role of implicit
  regularization in deep learning.
\newblock {\em arXiv preprint arXiv:1412.6614}, 2014.

\bibitem{rahimi2008random}
Ali Rahimi and Benjamin Recht.
\newblock Random features for large-scale kernel machines.
\newblock In {\em Advances in neural information processing systems}, pages
  1177--1184, 2008.

\bibitem{rahimi2008uniform}
Ali Rahimi and Benjamin Recht.
\newblock Uniform approximation of functions with random bases.
\newblock In {\em 2008 46th Annual Allerton Conference on Communication,
  Control, and Computing}, pages 555--561. IEEE, 2008.

\bibitem{rahimi2009weighted}
Ali Rahimi and Benjamin Recht.
\newblock Weighted sums of random kitchen sinks: Replacing minimization with
  randomization in learning.
\newblock In {\em Advances in neural information processing systems}, pages
  1313--1320, 2009.

\bibitem{rotskoff2018parameters}
Grant Rotskoff and Eric Vanden-Eijnden.
\newblock Parameters as interacting particles: long time convergence and
  asymptotic error scaling of neural networks.
\newblock In {\em Advances in neural information processing systems}, pages
  7146--7155, 2018.

\bibitem{shalev2014understanding}
Shai Shalev-Shwartz and Shai Ben-David.
\newblock {\em Understanding machine learning: From theory to algorithms}.
\newblock Cambridge university press, 2014.

\bibitem{sirignano2018mean}
Justin Sirignano and Konstantinos Spiliopoulos.
\newblock Mean field analysis of neural networks: A central limit theorem.
\newblock {\em arXiv preprint arXiv:1808.09372}, 2018.

\bibitem{song2018mean}
Mei Song, A~Montanari, and P~Nguyen.
\newblock A mean field view of the landscape of two-layers neural networks.
\newblock In {\em Proceedings of the National Academy of Sciences}, volume 115,
  pages E7665--E7671, 2018.

\bibitem{xie2016diverse}
Bo~Xie, Yingyu Liang, and Le~Song.
\newblock Diverse neural network learns true target functions.
\newblock In {\em Artificial Intelligence and Statistics}, pages 1216--1224,
  2017.

\bibitem{zhang2016understanding}
Chiyuan Zhang, Samy Bengio, Moritz Hardt, Benjamin Recht, and Oriol Vinyals.
\newblock Understanding deep learning requires rethinking generalization.
\newblock In {\em International Conference on Learning Representations}, 2017.

\bibitem{zou2018stochastic}
Difan Zou, Yuan Cao, Dongruo Zhou, and Quanquan Gu.
\newblock Stochastic gradient descent optimizes over-parameterized deep {ReLU}
  networks.
\newblock {\em arXiv preprint arXiv:1811.08888}, 2018.

\end{thebibliography}

\begin{appendix}
\section{Proof of Lemma~\ref{lm:a_star}}
\label{sec: approx}
\begin{proof}
For any $B_0$, let $\ba^*(B_0)=\{a^*(\bb_k^0)/m\}_{k=1}^m$, where $a^*$ is the function defined in Assumption~\ref{assump:early_stop}. Let 
\begin{equation}
f(\bx;A^*(B_0),B_0)=\sum\limits_{k=1}^m a^*(\bb_k^0)\sigma(\bx^T\bb_k^0).
\end{equation}
Then we have $\EE_{B_0}f(\bx;A^*(B_0),B_0)=f^*(\bx)$. Now, consider 
\begin{equation}
Z(B_0)=\sqrt{\EE_{\bx} (f(\bx;A^*(B_0),B_0)-f^*(\bx))^2},
\end{equation}
then if $\tilde{B}_0$ is different from $B_0$ at only one $\bb_k^0$, we have
\begin{equation}
|Z(B_0)-Z(\tilde{B}_0)|\leq\frac{2\gamma}{m}.
\end{equation}
Hence, by McDiarmid's inequality, for any $\delta>0$, with probability no less than $1-\delta$, we have
\begin{equation}
Z(B_0)\leq\EE Z(B_0)+\gamma\sqrt{\frac{2\ln(1/\delta)}{m}}.
\end{equation}
On the other hand,
\begin{equation}
\EE Z(B_0)\leq\sqrt{\EE Z^2(B_0)}=\sqrt{\EE_{\bx} Var(f(\bx;A^*(B_0),B_0))}\leq\frac{\gamma}{\sqrt{m}}.
\end{equation}
Therefore, we have
\begin{equation}
\cR(\ba^*,B_0)=Z^2(B_0)\leq\frac{\gamma^2}{m}\left(1+\sqrt{2\ln(\frac{1}{\delta})}\right)^2.
\end{equation}
Finally, by Assumption~\ref{assump:early_stop}, $\|\ba^*\|\leq\frac{\gamma}{\sqrt{m}}$.
\end{proof}

\section{Proof of Lemma~\ref{lm:rad}}
\label{sec: rf-rad}
For any $Q>0$, let $\cF_Q=\{f( \cdot;\ba,B_0):\ \|\ba\|\leq Q\}$. 
We can bound the Rademacher complexity of $\cF_Q$ as follows.
\begin{align}
\rad(\cF_Q) &= \frac{1}{n}\EE_\xi [\sup_{\|\ba\|\leq Q}\sum\limits_{i=1}^n\xi_i \sum\limits_{k=1}^m a_k\sigma(\bx_i^T\bb_k^0)] \nonumber \\
  &\leq \frac{1}{n}\EE_\xi [\sup_{\|\ba\|\leq Q, \|\bb_k^0\|\leq1} \sum\limits_{k=1}^m a_k \sum\limits_{i=1}^n \xi_i\sigma(\bx_i^T\bb_k^0)] \nonumber \\
  &= \sup_{\|\ba\|\leq Q} \sum\limits_{k=1}^m a_k \frac{1}{n}\EE_\xi[\sup_{\|\bb_k^0\|\leq1} \sum\limits_{i=1}^n \xi_i\sigma(\bx_i^T\bb_k^0)] \nonumber \\
  &\leq  Q \sqrt{\frac{m}{n}\EE_\xi\sup_{\|\bb_k^0\|\leq1} \sum\limits_{i=1}^n \xi_i\sigma(\bx_i^T\bb_k^0)} \nonumber \\
  &= Q\sqrt{m \,\rad(\{\sigma(\bb^T\bx):\ \|\bb\|\leq1\})} \nonumber \\
  &\leq \sqrt{m}Q.
\end{align}
Next, let $\cH_Q=\{(f(\cdot;\ba,B_0)-f^*)^2:\ \|\ba\|\leq Q\}$. Since $|f^*(\bx)|\leq1$  for any $\bx$,  by the Cauchy-Schwartz 
inequality, $|f(\bx;\ba,B_0)|\leq\sqrt{m}Q$.  Hence we can bound the Rademacher complexity of $\cH_Q$ by
\begin{equation}
\rad(\cH_Q)\leq 2(\sqrt{m}Q+1)\rad(\cF_Q)\leq2mQ^2+2\sqrt{m}Q,
\end{equation}
using that $(f(\cdot;\ba,B_0)-f^*)^2$ is Lipschitz continuous with Lipschitz constant bounded by $2\sqrt{m}Q+1$. Therefore, for any $\delta>0$, with probability larger than $1-\delta$, we have
\begin{equation}\label{eq:rad}
\left|\cR(\ba,B_0)-\hat{\cR}_n(\ba,B_0)\right|\leq \frac{2mQ^2+2\sqrt{m}Q}{\sqrt{n}}+(\sqrt{m}Q+1)^2\sqrt{\frac{2\ln(1/\delta)}{n}},
\end{equation}
for any $\ba$ with $\|\ba\|\leq Q$.

Finally, for any integer $k$, let $Q_k=2^{k}$ and $\delta_k=2^{-|k|}\delta$. Then, with probability larger than 
\begin{equation}
1-\sum\limits_{k=-\infty}^\infty \delta_k \geq 1-3\delta, 
\end{equation}
we have that~\eqref{eq:rad} holds for all $Q=Q_k$. 
Given  any $\ba \in \RR^m$, we can find a $Q_k$ such that $Q_k\leq2\|\ba\|$, which means
\begin{align*}
\left|\cR(\ba,B_0)-\hat{\cR}_n(\ba,B_0)\right| &\leq \frac{8m\|\ba\|^2+4\sqrt{m}\|\ba\|}{\sqrt{n}}+(2\sqrt{m}\|\ba\|+1)^2\sqrt{\frac{2\ln(1/\delta_k)}{n}} \nonumber \\
 &\leq \frac{2(2\sqrt{m}\|\ba\|+1)^2}{\sqrt{n}}\left(1+\sqrt{2\ln(\frac{2}{\delta}(\|\ba\|+\frac{1}{\|\ba\|}))}\right).
\end{align*}
This completes the proof.

\section{Proof of Lemma~\ref{lem: init-risk}}
\label{sec: init-risk}
\begin{proof}
Define $\cF = \{h(a,\bb)=a\sigma(\bb^T\bx)\,:\, \|\bx\| \leq 1\}$. By the standard Rademacher complexity bound (see Theorem 26.5 of \cite{shalev2014understanding}), we have, with probability at least $1-\delta$,
\begin{align*}
\sup_{\|\bx\|\leq 1} |\frac{1}{m}\sum_{k=1}^m a_k\sigma(\bb_k^T\bx)-0|\leq 2\rad_m(\cF) + \beta \sqrt{\frac{\ln(1/\delta)}{m}}.
\end{align*}
Moreover, since $\phi_k(\cdot)\Def a_k \sigma(\cdot)$ is $\beta-$Lipschitz continuous, by applying the contraction property of  Rademacher complexity (see Lemma 26.9 of \cite{shalev2014understanding}) we have 
\begin{align*}
\rad_m(\cF)&=\frac{1}{m}\EE_{\varepsilon}[\sup_{\|\bx\|\leq 1}\sum_{k=1}^m \varepsilon_k a_k \sigma(\bb_k^T\bx)]\\
&\leq \frac{\beta}{m}\EE_{\varepsilon}[\sup_{\|\bx\|\leq 1}\sum_{k=1}^m \varepsilon_k \bb_k^T\bx] \\
&\leq \frac{\beta}{\sqrt{m}},
\end{align*}
where the last inequality follows from the Lemma 26.10 of \cite{shalev2014understanding}.
Thus with probability $1-\delta$, we have that for any $\|\bx\|=1$, 
\[
    |f(\bx;\Theta_0)| = m |\frac{1}{m}\sum_{k=1}^m a_k\sigma(\bb_k^T\bx)| \leq \sqrt{m}\beta(2+\sqrt{\ln(1/\delta)}).
\]
Thus $\hcR(\Theta_0)\leq \frac{1}{2n}\sum_{i=1}^n(1+|f(\bx_i;\Theta_0)|)^2\leq \frac{1}{2}(1+\sqrt{m}\beta(2+\sqrt{\ln(1/\delta)}))^2$.
\end{proof}

\section{Proof of Lemma~\ref{lem: gram-init}}
\label{sec: gram-init}
\begin{proof}
For a given $\varepsilon\geq 0$, define events
\begin{align*}
S^a_{i,j} &=\{\Theta_0: |G^a_{i,j}(\Theta_0)-\frac{1}{n}k^a(\bx_i,\bx_j)|\leq \varepsilon/n\}\\
S^b_{i,j} &=\{\Theta_0: |G^b_{i,j}(\Theta_0)-\frac{1}{n}k^b(\bx_i,\bx_j)|\leq \varepsilon/n\}.
\end{align*}
 Hoeffding's inequality gives us that 
\[
\PP[S^a_{i,j}]\geq 1- e^{-2m\varepsilon^2}, \quad \PP[S^b_{i,j}]\geq 1- e^{-2m\varepsilon^2}.
\]
Thus with probability at least $(1-e^{-2m\varepsilon^2})^{2n^2}\geq 1-2n^2e^{-2m\varepsilon^2}$, we have 
\[
    \max\{\|G^a-K^a\|_F, \|G^b-K^b\|_F\} \leq \varepsilon.
\]
Using Weyl's Theorem, we have
\begin{align*}
\lambda_{\min}(G(\Theta_0)) &\geq \lambda_{\min}(G^a) + \beta^2\lambda_{\min}(G^b)\\
&\geq \lambda^a_n - \|G^a-K^a\|_F + \beta^2\left( \lambda^b_n - \|G^b-K^b\|_F\right) \\
&\geq \lambda^a_n +\beta^2 \lambda^b_n - (1+\beta^2)\varepsilon.
\end{align*}
Taking $\varepsilon=\lambda_n/4$, we complete the proof.
\end{proof} 

\end{appendix}
\end{document}